\documentclass{llncs}
\usepackage{graphicx}
\usepackage{adjustbox}
\usepackage{algorithm,algpseudocode}
\usepackage{multirow}
\usepackage[flushleft]{threeparttable}
\usepackage{amsmath,amssymb,amscd,amstext}
\DeclareMathOperator{\dom}{dom}

\usepackage{caption}
\usepackage{subcaption}
\usepackage{threeparttable,booktabs}
\usepackage{tabularx}
\usepackage{bm}
\usepackage{pifont}
\usepackage{hyperref}
\usepackage{xcolor}

\DeclareMathOperator*{\argmax}{argmax}

\definecolor{blu}{rgb}{0,0,1}
\definecolor{gre}{rgb}{0,.5,0}
\definecolor{bonus}{rgb}{.9,.9,1}
\def\blu#1{{\color{blu}#1}}

\def\norm#1{\|#1\|}

\newtheorem{assumption}{Assumption}
\newcommand{\nn}{\nonumber}

\begin{document}

\title{Fast Convergence of Random Reshuffling under Over-Parameterization and the Polyak-\L ojasiewicz Condition}
%
%
\author{Chen Fan\inst{1} \and
Christos Thrampoulidis\inst{2} \and
Mark Schmidt\inst{1}}
%
%
\institute{Department of Computer Science \and
Department of Electrical and Computer Engineering \\
University of British Columbia \\  
Vancouver, British Columbia, Canada \\
\email{\{fanchen2,schmidtm\}@cs.ubc.ca,\{cthrampo@ece.ubc.ca\}}}
\maketitle              
%
\begin{abstract}
Modern machine learning models are often over-parameterized and as a result they can interpolate the training data. Under such a scenario, we study the convergence properties of a sampling-without-replacement variant of stochastic gradient descent (SGD) known as random reshuffling (RR). Unlike SGD that samples data with replacement at every iteration, RR chooses a random permutation of data at the beginning of each epoch and each iteration chooses the next sample from the permutation. For under-parameterized models, it has been shown RR can converge faster than SGD under certain assumptions. 
However, previous works do not show that RR outperforms SGD in  over-parameterized settings except in some highly-restrictive scenarios. For the class of Polyak-\L ojasiewicz (PL) functions, we show that RR can outperform SGD in over-parameterized settings when either one of the following holds: (i) the number of samples ($n$) is less than the product of the condition number ($\kappa$) and the parameter ($\alpha$) of a weak growth condition (WGC), or (ii) $n$ is less than the parameter ($\rho$) of a strong growth condition (SGC). 

\end{abstract}
\section{Introduction}
We consider finite-sum minimization problems of the form 
\begin{align}
    \min \biggl\{ f(x) = \frac{1}{n}\sum_{i=1}^n f(x;i) \biggr\} \label{eq:p}. 
\end{align}
Stochastic gradient descent (SGD) is a popular algorithm for solving machine learning problems of this form. A significant amount of effort has been made to understand its theoretical and empirical properties \cite{bottou2018optimization}.
SGD has a simple update rule in which a sample $i_k$ is chosen randomly with replacement and at each iteration we compute $x^{k+1} = x^k - \eta^k \nabla f(x^k;i_k)$. This is cheaper than using the full gradient at each iteration.
However, it is well known that the convergence rate of SGD 
can be much worse than the convergence rate of full-gradient descent. For example, for strongly-convex functions SGD has a sublinear convergence rate while full-gradient descent has a linear convergence rate.

Given the increasing complexity of modern learning models, a practically-relevant question to ask is how SGD performs in over-parameterized settings, under which the model interpolates (exactly fits) the data. 
Previously, it has been shown that SGD can achieve a linear convergence rate like full-gradient descent under various interpolation conditions for strongly-convex functions~\cite{moulines2011non,schmidt2013fast,needell2014stochastic,vaswani2019fast} such as the strong growth condition (SGC) and the weak growth condition (WGC).
An assumption that is weaker than strong convexity which allows full gradient descent to achieve a linear rate is the Polyak-\L ojasiewicz (PL) condition~\cite{polyak1963gradient}.
Recently, the PL condition has gained popularity in machine learning~\cite{karimi2016linear} and it has been shown that several overparameterized models that interpolate the data satisfy the PL condition \cite{bassily2018exponential}. Similar to the strongly-convex case, under interpolation and the PL condition SGD can achieve a linear rate similar to full gradient descent \cite{bassily2018exponential,vaswani2019fast}.

A popular variation on SGD is random reshuffling (RR).
At each epoch $t \in \{1,2,...,T\}$, the algorithm RR randomly chooses a permutation $\pi^t$. 
That is,  $\pi^t_{j+1}$ is sampled without replacement from the set $\{1,2,...,n\}$ for $j \in \{0,1,...,n-1\}$. Then it performs the following update, going through the dataset in the order specified by $\pi^t$ 
\begin{align}
    x_{j+1}^t = x_j^t - \eta_j^t\nabla f(x_j^t; \pi^t_{j+1}). \label{eq:update2}
\end{align}
Note that $x_0\triangleq x^1_0$ is the initialization and $x_0^{t+1}=x_n^t$ $\forall t\geq 1$. We summarize the method in Algorithm~\ref{alg:RR}.
A variation on RR is the incremental gradient (IG) method where $\pi^t$ is deterministic and fixed over all epochs.

\begin{algorithm}[H] 
   \caption{Random Reshuffling (RR)}
    \textbf{Input:} $x_0$, $T$, step sizes $\{ \eta_j^t\}$    
\begin{algorithmic}[1] \label{alg:rr}
  \For{$t = 1, 2, \ldots, T$}
    \\
    \quad Choose a permutation $\pi^t$ from the set of all permutations; set $x_0^t = x_0$ if $t=1$; set $x_0^t = x_n^{t-1}$ if $t > 1$.  
     \For{$j = 1,2, \ldots, n-1$}
     \State $x_{j+1}^t = x_j^t - \eta_j^t\nabla f(x_j^t; \pi^t_{j+1})$ \label{eq:update1} 
 \EndFor
\EndFor
\end{algorithmic}
\label{alg:RR}
\end{algorithm}

 
RR has long been known to converge faster than SGD empirically for certain problems~\cite{bottou2009curiously,bottou2012stochastic}. However, analyzing RR is more difficult than SGD because (conditioned on the past iterates) each individual gradient is no longer an unbiased estimate of the full gradient. Thus, the analysis of RR has only emerged in a series of recent efforts~\cite{haochen2019random,nagaraj2019sgd,safran2020good,gurbuzbalaban2021random}. Previous works have shown that RR outperforms SGD
for strongly-convex objectives in various under-parameterized settings, when the the number of epochs ($T$) is sufficiently large. 
However, these sublinear rates for RR are slower than the linear rates obtained for SGD in the over-parameterized setting. Further, current convergence rate analyses for RR in the over-parameterized setting either make unrealistic assumptions (see the next section) or are also slower than SGD unless we make very-strong assumptions.

In this work, we analyze RR and IG for PL functions under both the SGC and WGC over-parameterized assumptions.
We give explicit convergence rates for RR that can be faster than the best known rate for SGD, under realistic assumptions that hold for situations like over-parameterized least squares.
Our results also show IG can converge faster than SGD (though not RR) in some settings. We consider relaxations of the SGC and WGC where interpolation is only approximately satisfied, showing faster convergence of RR to a neighbourhood of the solution.


\section{Related Work}
\paragraph{Optimization under the PL condition:} the PL inequality 
was first explored by Polyak~\cite{polyak1963gradient} and \L{}ojasiewicz~\cite{lojasiewicz1963topological}. It applies to a wide range of important machine learning problems such as least square and logistic regression (over a compact set) \cite{karimi2016linear}. More generally, any function of the form $f(x) = g(Ax)$ for a matrix $A$ with a $\mu$-strongly convex funciton $g$ satisfies the $\mu$-PL condition \cite{karimi2016linear}. 
Several recent works have argued that considering a local PL condition around minimizers can be used as a model for 
analyzing the effectiveness of SGD in training neural networks \cite{du2018gradient,liu2022loss,oymak2019overparameterized,soltanolkotabi2018theoretical}.

Polyak~\cite{polyak1963gradient} showed that full gradient descent on smooth objective functions can achieve a linear convergence rate under the PL condition. But it has recently been highlighted that the PL condition can be used to show linear convergence rates for a variety of methods \cite{karimi2016linear}. Typically, the PL condition leads to similar convergence rates as those obtained under the stronger condition of strong convexity. In the case of SGD under interpolation, it has been shown that the rate of SGD under the PL condition is linear \cite{bassily2018exponential,vaswani2019fast}. However, the convergence rates for SGD under interpolation for $\mu$-PL functions are slower than those for strongly convex functions~(see Table \ref{table:summ1}).


\paragraph{RR for Strongly-Convex Under-Parameterized Problems:} 
in this paragraph we focus our discussion on the case of strongly-convex functions, the subject of most literature on the topic. 
Bottou conjectured that the convergence rate of RR is~$O(\frac{1}{n^2 T^2})$ \cite{bottou2009curiously}, where $T$ is the number of epochs. But it was several years before progress was made at showing this.
The difficulty in the analysis of RR arises because of the bias in the conditional expectation of the gradients, 
\begin{align}
    \mathbb{E} \bigr[\nabla f(x^t_i;\pi^t_{i+1}) \;|\; x_i^t \bigl] \neq \nabla f(x_i^t). 
\end{align}
An early attempt to analyze RR~\cite{recht2012toward} was not successful because their noncommutative arithmetic-geometric mean inequality conjecture was proven to be false \cite{lai2020recht}. An~$\tilde{\mathcal{O}}(\frac{1}{n^2 T^2})$ rate was first shown asymptotically by G\"urb\"uzbalaban et al. \cite{gurbuzbalaban2021random}
Haochen and Sra~\cite{haochen2019random} give the first non-asymptotic convergence result of $\tilde{\mathcal{O}}(\frac{1}{\mu^4}(\frac{1}{n^2T^2}+\frac{1}{T^3}))$ with a strong-convexity constant $\mu$, under strong assumptions. Under weaker assumptions a rate of 
$\tilde{\mathcal{O}}(\frac{1}{\mu^3nT^2})$ was shown when $T\gtrsim \frac{1}{\mu^2}$ by assuming component-wise convexity of each $f(\cdot;i)$~\cite{nagaraj2019sgd}, matching a lower bound of $\Omega(\frac{1}{nT^2})$~\cite{rajput2020closing}.
\footnote{Following previous conventions in the literature \cite{nagaraj2019sgd,safran2021random,cha2023tighter}, we take $\kappa = \Theta(\frac{1}{\mu})$ for comparisons.}
Note that this rate is faster than the $\tilde{\mathcal{O}}(\frac{1}{nT})$ rate of SGD when the large epoch requirement is satisfied. Mischenko et al.~\cite{mishchenko2020random} obtain the same rate of $\tilde{\mathcal{O}}(\frac{1}{\mu^3 nT^2})$ but only require  $T\gtrsim \frac{1}{\mu}$. Their analysis is dependent on the underlying component-wise convexity structure. Ahn et al.~\cite{ahn2020sgd} remove this dependence and obtain the same rate with $T \gtrsim \frac{1}{\mu}$. However, their analysis relies on each $f(\cdot;i)$ being $G$-Lipschitz ($\lVert \nabla f(\cdot;i)\rVert \leq G$ for all $i$), which may require a constraint on problem \eqref{eq:p} and a projection operation is needed to ensure the iterates are bounded \cite{ahn2020sgd,nguyen2021unified}. Nguyen et al.~\cite{nguyen2021unified}  give a unified analysis for shuffling schemes other than RR, while Safran and Shamir~\cite{safran2020good} show a lower bound of $\Omega(\frac{1}{\mu n^2T^2}+\frac{1}{\mu nT^3})$ when $f$ is a sum of $n$ quadratics. In a more recent work, they have shown that RR does not significantly improve over SGD unless $T$ is larger than $\Theta(\frac{1}{\mu})$ in the worst case \cite{safran2021random}. We can also consider other shuffling schemes; Lu et al. \cite{lu2023grab} design the GraB algorithm that achieves a faster rate of $\mathcal{\Tilde{O}}(\frac{1}{\mu^3 n^2 T^2})$ by incorporating previous gradient information into determining the current permutation. Furthermore, the matching lower bound for this algorithm is given by Cha et al.~\cite{cha2023tighter}.
We note that all of these sublinear convergence rates are slower than the linear rates that are possible for SGD on over-parameterized problems.

\paragraph{RR for Strongly-Convex Over-Parameterized Problems:} despite the widespread use of RR for training over-parameterized models, there is relatively little literature analyzing this setting. Haochen and Sra~\cite{haochen2019random} show that the convergence rate of RR is at least as fast as SGD under interpolation even without any epoch requirements. However, their result does not show that a gap in the rates can exist and only applies in the degenerate case where each function is strongly-convex.\footnote{In this setting we could solve the problem by simply applying gradient descent to any individual function and ignoring all other training examples.} Ma and Zhou~\cite{ma2020understanding} show a faster rate than SGD for RR under over-parameterization and a ``weak strong convexity'' assumption, but their result is also somewhat degenerate as their assumption excludes standard problems like over-parameterized least squares.\footnote{The usual ``weak strong convexity'' assumption is that for each $f_i$ the strong convexity inequality holds for the projection onto the minima with respect to $f_i$ which holds for least squares. But Ma and Zhou instead use the projection onto the intersection of this set with the set of minima with respect to $f$, which does not hold in general for least squares}
We can obtain interpolation results under more-realistic assumptions as special cases of the results of Mischenko et al.~\cite{mishchenko2020random} and Nguyen et al.~\cite{nguyen2021unified}. However, in the interpolation setting the rates obtained by these works for convex and strongly-convex functions are slower than the rate obtained by Vaswani et al.~\cite{vaswani2019fast} for SGD~(we summarize the known linear rates under over-parameterization in Table \ref{table:summ1}). 

\paragraph{RR for Polyak-\L{}ojasiewicz Problems:}
several works have presented convergence rates for RR on PL problems~\cite{haochen2019random,ahn2020sgd,mishchenko2020random,nguyen2021unified}.
Haochen and Sra give the first non-asymptotic rate of $\tilde{\mathcal{O}}(\frac{1}{\mu^4}(\frac{1}{n^2T^2}+\frac{1}{T^3}))$ for the under-parameterized setting but requiring each individual function to have Lipschitz continuous Hessian \cite{haochen2019random}. 
Ahn et al. improve the rate to $\mathcal{\tilde{O}}(\frac{1}{\mu^3nT^2})$ when $T \gtrsim \frac{1}{\mu}$ assuming bounded gradients~\cite{ahn2020sgd} but these sublinear rates for the under-parameterized setting are slower than the linear rates of SGD for over-parameterized problems.
The result of Mischenko et al. on $\mu$-PL functions~\cite{mishchenko2020random} can be used to obtain a linear rate under over-parameterization,\footnote{By taking $B=0$ in Mischenko et al.'s Theorem~4.} but only in the degenerate setting where all functions have the same gradient. The result of Nguyen et al. on PL functions~\cite[Theorem~1]{nguyen2021unified} concerns the under-parameterized setting. But in Appendix~\ref{sec:literature_app} we show how the proof of Theorem~1 of Nguyen et al. \cite{nguyen2021unified} can be modified to give a linear convergence rate of RR for the over-parameterized setting. This RR rate can be faster than SGD, but under a more restricted scenario than the results presented in this work (we discuss the precise details in Section~\ref{sec:contributions}).

\section{Assumptions}
In this section, we present the assumptions made in our analyses.
Our first assumption is that the individual and overall functions are smooth and bounded below.
\begin{assumption}[Smoothness and Boundedness]\label{smooth}
The objective $f$ is $L$-smooth and each individual loss $f(\cdot;i)$ is $L_i$-smooth such that $\forall x,x' \in \dom(f)$,
\begin{align}
    \norm{\nabla f(x) - \nabla f(x')} & \leq L \norm{x - x'},\\
   \lVert \nabla f(x;i) - \nabla f(x';i) \rVert &\leq L_i \lVert x - x'\rVert, \forall i. 
\end{align}
We denote $L_{\max} \triangleq \max\limits_{i} L_i$. We assume $f$ is lower bounded by $f^*$, which is achieved at some $x^*$, so $f^* = f(x^*)$. We also assume that each $f(\cdot;i)$ is lower bounded by some $f_i^*$.
\end{assumption} 
Next, we give the definition of the Polyak-\L ojasiewicz (PL) inequality which lower bounds the size of the gradient as the value of $f$ increases.
\begin{assumption}[$\mu$-PL]\label{def_pl}
There exists some $\mu > 0$ such that
\begin{align}
    f(x) - f^* \leq \frac{1}{2\mu}\lVert \nabla f(x) \rVert^2 \quad \forall x \in \dom(f), 
\end{align}
where $f^*$ is the optimal value of $f$. 
\end{assumption}
 All strongly-convex functions satisfy the PL inequality. However, many important functions satisfy the PL inequality that are not necessarily strongly-convex including the classic example of the squared prediction error with a linear prediction function~\cite{karimi2016linear}. 
 PL functions can be non-convex but the condition implies invexity, a generalization of convexity. For smooth functions invexity is equivalent to every stationary point of $f$ being a global minimum (though unlike strongly-convex functions we may have multiple stationary points)~\cite{craven1985invex}.
  Recent literature shows that the PL condition and variants of it can be used to help analyze a variety of complicated models~\cite{bassily2018exponential,du2018gradient,liu2022loss,oymak2019overparameterized,soltanolkotabi2018theoretical}.



Next, we formally define interpolation which is the key property of the optimization problem implied by over-parameterization.
\begin{assumption}[Interpolation]\label{interp}
    We are in the interpolating regime, which we take to mean that
    \begin{align}
        \nabla f (x^*) = 0 \quad \Longrightarrow \quad \nabla f(x^*;i) = 0. 
    \end{align}
\end{assumption}
Thus, by interpolation we mean that stationary points with respect to the function $f$ are also stationary points with respect to the individual functions $f(\cdot;i)$. However, we do not assume that $x^*$ is unique. 
Our results also consider either the strong 
or weak growth conditions~\cite{schmidt2013fast,vaswani2019fast} regarding how the individual gradients change as $x$ moves away from  $x^*$.
\begin{assumption}[SGC]\label{SGC}
There exists a constant $\rho \geq 1$ such that the following holds 
\begin{align}\label{eq:SGC}
    \frac{1}{n}\sum_{i=1}^n \lVert \nabla f(x;i)\rVert^2 \leq \rho \lVert \nabla f(x) \rVert^2, \quad \forall x \in \dom(f). 
\end{align}
\end{assumption}
\begin{assumption}[WGC]\label{WGC}
There exists a constant $\alpha \geq 0$ such that the following holds 
\begin{align}\label{eq:SGC}
    \frac{1}{n}\sum_{i=1}^n \lVert \nabla f(x;i)\rVert^2 \leq 2 \alpha L (f(x) - f(x^*)), \quad \forall x \in \dom(f), 
\end{align}
where $f$ is lower bounded by $f^*$ and $f^* = f(x^*)$.
\end{assumption}
There are a close relationships between interpolation, the SGC, and the WGC. The SGC implies interpolation, while for smooth functions the SGC implies the WGC~\cite[Lemma~5]{mishkin2020interpolation}.
Further, for functions satisfying smoothness and $\mu$-PL~(Assumptions~\ref{def_pl} and~\ref{smooth}) all three conditions are equivalent if the $f_i$ are invex (SGC implies interpolation which implies WGC which implies SGC). In this setting of smooth $\mu$-PL functions with invex $f_i$, 
the constant $\alpha$ in the WGC is upper bounded by~$\frac{L_{\max}}{L}$ and by~$\rho$~\cite[Lemmas~5-6]{mishkin2020interpolation},\footnote{We include the $\alpha \leq \rho$ result in Appendix~\ref{app:SGCWGC} since it is shown under stronger assumptions~\cite{vaswani2019fast} or stated differently~\cite{mishchenko2020random} in prior work.}
while the value of the constant $\rho$ in the SGC is upper bounded by~$\frac{\alpha L}{\mu}$~\cite[Proposition~1]{vaswani2019fast} and is thus also bounded by~$\frac{L_{\max}}{\mu}$. The value of $\rho$ in the SGC is lower bounded by 1 while in Appendix~\ref{app:SGCWGC} we show that $\alpha$ in the WGC is lower-bounded by $\frac{\mu}{L}$.

A classical setting where we would expect these interpolation properties to hold is binary classification with a linear classifier when the data is linearly separable. In this setting the gradient of each example can be made to converge to zero due to the linear separability (under standard loss functions).
In Appendix~\ref{app:SGCWGC}, we give a generalization of an existing result~\cite[Lemma~1]{vaswani2019fast} showing that the SGC holds for a class of functions including the squared hinge loss and logistic regression in the linearly-separable setting.

In the most-general form of our results, we consider relaxations of the SGC and the WGC that do not require the data to be fit exactly \cite{poljak1973pseudogradient,cevher2019linear}. 
\begin{assumption}\label{SGC_noise}
There exists constants $\rho \geq 0$ and $\sigma \geq 0$ such that the following holds: $\forall x \in \dom(f),$
\begin{align}
    \frac{1}{n}\sum_{i=1}^n \lVert \nabla f(x;i)\rVert^2 \leq \rho \lVert \nabla f(x) \rVert^2 + \sigma^2.
\end{align}
\end{assumption}
\begin{assumption}\label{WGC_noise}
There exists constants $\alpha \geq 0$ and $\sigma \geq 0$ such that the following holds: $\forall x \in \dom(f),$
\begin{align}
    \frac{1}{n}\sum_{i=1}^n \lVert \nabla f(x;i)\rVert^2 \leq 2 \alpha L(f(x)-f(x^*)) + \sigma^2.
\end{align}
\end{assumption}
These assumptions reduce to the SGC and WGC when $\sigma=0$, and we note that the relaxed WGC is related to the expected smoothness condition of Gower et al.~\cite{gower2019sgd}. We also note that Assumption \ref{SGC_noise} reduces to the bounded variance assumption $\frac{1}{n}\sum_{i=1}^n\lVert \nabla f(x;i) - \nabla f(x) \rVert^2 \leq \sigma^2$ when $\rho = 1$, which is commonly used in the analysis of SGD \cite{bottou2018optimization}. 

\section{Contributions} \label{sec:contributions}
\begin{table}[t!]
\centering
\begin{threeparttable} 
\caption{Number of gradient evaluations required by each algorithm to obtain an $\epsilon$-accurate solution, which is defined as $\lVert x - x^*\rVert^2 \leq \epsilon$ and $f(x)-f(x^*) \leq \epsilon$ for $\mu$-strongly convex and $\mu$-PL objectives respectively. The contributions of this work are the rates highlighted in \blu{blue}.\tnote{(1)}\tnote{(2)}\tnote{(3)}}  
  \label{table:summ1}
\begin{tabularx}{1.0\textwidth}{ccccc}
\hline
                              & Interp                                  & WGC                                                  & SGC  
                              & Interp+Invex\\ \hline
SGD/SC & -               & $\alpha \frac{L}{\mu} $                            \cite{vaswani2019fast} &   $\rho \frac{L}{\mu}$   \cite{schmidt2013fast}   & $\frac{L_{\max}} {\mu}$\cite{needell2014stochastic}  \\
RR/SC  & - & - & - & $\frac{L_{\max}}{\mu}n$ \cite{mishchenko2020random}\\ \hline
SGD/PL           & $\frac{L^2_{\max}}{\mu^2}$ \cite{bassily2018exponential}  & \blu{$\alpha \frac{L^2}{\mu^2} $}                         & $\rho \frac{L}{\mu}$     \cite{vaswani2019fast} & $\frac{L^2_{\max}}{\mu^2}$ \cite{bassily2018exponential}  \\
IG/PL             & $\frac{L^2_{\max}}{\mu^2}n$ \cite{nguyen2021unified} & \blu{$\frac{L_{\max}}{\mu} n \sqrt{\frac{\alpha L}{\mu}}$} &             \blu{$\frac{L_{\max}}{\mu} n \sqrt{\rho}$}      & \blu{$(\frac{L_{\max}}{\mu})^{3/2} n $}        \\ 
RR/PL           &     $\frac{L_{\max}^2}{\mu^2}n$ \cite{nguyen2021unified}   &  \blu{$ \frac{L_{\max}}{\mu}\sqrt{n}(\sqrt{n} \vee \sqrt{\frac{\alpha L}{\mu}} )$}        &  \blu{$\frac{L_{\max}}{\mu} \sqrt{n} (  \sqrt{n}\vee \sqrt{\rho})$}  & \blu{$\frac{L_{\max}}{\mu}\sqrt{n} ( \sqrt{n}\vee \sqrt{\frac{L_{\max}}{\mu}})$}
\\\hline
\end{tabularx}
\begin{tablenotes}\footnotesize
\item[(1)] We ignore numerical constants and logarithmic factors. Note that $\mu$-SC stands for $\mu$-strongly convex. Interp, SGC, and WGC refer to Assumption \ref{interp}, \ref{SGC}, and \ref{WGC} (respectively). 
Note that $L \leq L_{\max}$, for the WGC constant $\alpha$ we have $\frac{\mu}{L} \leq \alpha \leq \rho$~\cite[Appendix~\ref{app:SGCWGC}]{mishkin2020interpolation}, while for the SGC constant $\rho$ we have $\max\{1,\alpha\} \leq \rho \leq \frac{\alpha L}{\mu}$~\cite{vaswani2019fast}, and if all functions are invex we also have $\alpha \leq \frac{L_{\max}}{L}$ and $\rho \leq \frac{L_{\max}}{\mu}$~\cite{mishkin2020interpolation}. 
 \item[(2)] The ``Interp+Invex'' results assume each function is invex, and the strongly-convex results~\cite{needell2014stochastic,mishchenko2020random} additionally assume each function is convex.
 \item[(3)] The symbol ``$\vee$" refers to taking the maximum of two values. 
\end{tablenotes}
\end{threeparttable}
\end{table}

 Table \ref{table:summ1} summarizes known and new convergence results for the convergence of SGD on strongly-convex and PL functions under the interpolation assumptions (Assumptions ~\ref{interp}, \ref{SGC}, and~\ref{WGC}). In this table we see that existing RR rates for strongly-convex functions are slower than for SGD. In contrast, the RR rates shown in this work can be faster than the SGD rates for $\mu$-PL objectives. Our main contributions are summarized as follows:


 \paragraph{RR under WGC:} for $\mu$-PL functions satisfying the WGC, we derive the sample complexity of RR to be $\tilde{\mathcal{O}}(\max \{ \frac{L_{\max}}{\mu}n, \frac{L_{\max}}{\mu} \sqrt{n} \sqrt{\frac{\alpha L}{\mu}})$. In comparison, the sample complexity of SGD in this case is $\tilde{\mathcal{O}}(\alpha \frac{L^2}{\mu^2})$ (we show this in Appendix~\ref{sec:sgd_proofs}). Hence, RR outperforms SGD when $\bm{\alpha \frac{L}{\mu}} \bm{>} \bm{n}$ and $L_{\max} \sim L$ without requiring a large number of epochs. Thus, for $\mu$-PL objectives satisfying the WGC, RR converges faster than SGD for sufficiently ill-conditioned problems when the Lipschitz constants are similar.

\vspace{2pt}
\paragraph{RR under SGC:}  under the SGC we give a sample complexity for RR of\\$\tilde{\mathcal{O}}(\max \{ \frac{L_{\max}}{\mu}n,\frac{L_{\max}}{\mu} \sqrt{n} \sqrt{\rho}\})$, which is better than the~$\tilde{\mathcal{O}}(\rho \frac{L}{\mu})$ of SGD provided $\bm{\rho} \bm{>} \bm{n}$ and $L_{\max} \sim L$. The situation $\bm{\rho} \bm{>} \bm{n}$ happens when there is a large amount of disagreements in the gradients. Note that the rates of SGC can be faster than the WGC when $\rho \leq \alpha \frac{L}{\mu}$. This can be satisfied for example when $\alpha \sim 1$ and $L_{\max} \sim L$.  

\vspace{2pt}
\paragraph{IG under WGC/SGC:} IG can also outperform SGD for $\mu$-PL objectives when $\bm{\rho} \bm{>} \bm{n^2}$ and $L_{\max} \sim L$ under the SGC, or when $\bm{\alpha \frac{L}{\mu}} \bm{>} \bm{n^2}$ and $L_{\max} \sim L$ under the WGC.  The rates of RR for $\mu$-PL objectives can be better than IG by a factor of $\sqrt{n}$, and cannot be worse because $\rho \geq 1$ and $\alpha \geq \frac{\mu}{L}$ (see Proposition \ref{prop:alpha_wgc} in Appendix). 

\vspace{2pt}
\paragraph{Experiments:} we conduct experiments on objectives satisfying our assumptions. We experimentally see that a condition like $\bm{\rho} \bm{>} \bm{n}$ is important for RR to achieve a faster rate than SGD. For over-parameterized problems we demonstrate that in practice RR can converge faster than SGD even in the early training phase, but that RR has less of an advantage as $n$ grows.

\paragraph{Comparisons with Existing Rates for $\mu$-PL Functions:}
Appendix~\ref{sec:literature_app} discusses how the proof of
Nguyen et al.~\cite[Theorem~1]{nguyen2021unified} can used to obtain convergence rates of RR on over-parameterized $\mu$-PL objectives.
Under interpolation (Assumption~\ref{interp}), the analysis of Nguyen et al. implies a rate for RR of  $\mathcal{\Tilde{O}}((\frac{L_{\max}}{\mu})^2 n)$, which is always worse than the SGD rate of $\mathcal{\Tilde{O}}((\frac{L_{\max}}{\mu})^2)$ and our RR rate under this assumption of $\mathcal{\Tilde{O}}( \max\{ \frac{L_{\max}}{\mu} n, (\frac{L_{\max}}{\mu})^{3/2} \sqrt{n} \})$ (although this rate requires assuming each function is invex).
On the other hand, under the SGC (Assumption~\ref{SGC}) the analysis of Nguyen et al. implies a rate of $\mathcal{\Tilde{O}}(\max \{ \frac{L_{\max}}{\mu}n,\frac{L_{\max}}{\mu^2}\sqrt{n}\sqrt{n+\rho-1} \})$, but this is worse than our SGC result of~$\mathcal{\Tilde{O}}(\max \{ \frac{L_{\max}}{\mu} n, \frac{L_{\max}}{\mu} \sqrt{n}\sqrt{\rho} \})$ and only faster than SGD under a more-narrow range of problem constant settings.

\section{Convergence Results}
 We present the convergence results of RR and IG for $\mu$-PL objectives under the SGC and the WGC. Below, we use a constant step size $\eta$. The proofs of Theorem \ref{thm:rr} and \ref{thm:ig} are given in Appendix \ref{sec:main_proof}. 

\begin{theorem}[RR + $\mu$-PL]\label{thm:rr}
    Suppose f satisfies Assumption \ref{smooth} and \ref{def_pl}, 
    then RR with a learning rate $\eta$ satisfying 
    \begin{align}
        \eta \leq \min \{ \frac{1}{2 nL_{\max}}, \frac{1}{2 \sqrt{2} L_{\max} \sqrt{n \rho}}\} \quad \text{under Assumption \ref{SGC_noise}},\nn
    \end{align}
    or 
    \begin{align}
        \eta \leq \min \{ \frac{1}{2 nL_{\max}}, \frac{1}{2 \sqrt{2} L_{\max}\sqrt{n} \sqrt{\frac{\alpha L}{\mu}}}\} \quad \text{under Assumption \ref{WGC_noise}},
    \end{align}
    achieves the following rate for $\mu$-PL objectives:
    \begin{align}
        \mathbb{E}[f(x_n^{T}) - f(x^*)] \leq (1-\frac{1}{4}n\mu\eta)^T (f(x^0)-f(x^*)) + \frac{4L_{\max}^2 \eta^2 n \sigma^2}{\mu}. \label{eq:thm_rr_main}
    \end{align}
\end{theorem}
\begin{remark} \label{remark_rr}
    Set $\sigma = 0$ in \eqref{eq:thm_rr_main} and using Lemma~\ref{lem_recur} in Appendix \ref{sec:usef_lemmas} to solve the recursion, the sample complexity under the SGC is $\mathcal{\Tilde{O}}(\max \{ \frac{L_{\max}}{\mu}n, \frac{L_{\max}}{\mu} \sqrt{n}\sqrt{\rho} \})$, and the sample complexity under the WGC is $\mathcal{\Tilde{O}}(\max \{ \frac{L_{\max}}{\mu}n, \frac{L_{\max}}{\mu} \sqrt{n}\sqrt{\frac{\alpha L}{\mu}} \})$. These rates are faster than  SGD under the corresponding assumption for $\mu$-PL objectives when either $n < \rho$ or $n < \alpha \frac{L}{\mu}$ holds and $L_{\max} \sim L$.     In cases where $L_{\max}$ is not similar to $L$, RR can still outperform SGD when $n < \frac{\rho L^2}{L_{\max}^2}$ is satisfied under the SGC or $n < \frac{\alpha L^3}{\mu L_{\max}^2}$ is satisfied under the WGC.  
\end{remark}
\begin{theorem}[IG + $\mu$-PL]\label{thm:ig}
    Suppose f satisfies Assumption \ref{smooth} and \ref{def_pl}, 
    then IG with a learning rate $\eta$ satisfying 
    \begin{align}
        \eta \leq \min\{ \frac{1}{\sqrt{2}nL_{\max}}, \frac{1}{2 L_{\max} n \sqrt{\rho}} \} \quad \text{under Assumption \ref{SGC_noise}},\nn
    \end{align}
    or 
    \begin{align}
        \eta \leq \min\{ \frac{1}{\sqrt{2}nL_{\max}},\frac{1}{2 L_{\max} n \sqrt{\frac{\alpha L}{\mu}}} \} \quad \text{under Assumption \ref{WGC_noise}},
    \end{align}
    achieves the following rate for $\mu$-PL objectives:
    \begin{align}
        f(x_n^{T}) - f(x^*) \leq (1-\frac{1}{2}n\mu\eta)^T (f(x^0)-f(x^*)) + \frac{2L_{\max}^2 \eta^2 n^2 \sigma^2}{\mu}.
    \end{align}
\end{theorem}
\begin{remark}
    Similar to Remark \ref{remark_rr}, we can follow the same approach to obtain the sample complexity for IG to be $\mathcal{\Tilde{O}}(\max \{\frac{L_{\max}}{\mu}n, \frac{L_{\max}}{\mu} n \sqrt{\rho})\})=\mathcal{\Tilde{O}}(\frac{L_{\max}}{\mu} n \sqrt{\rho})$ under the SGC and $\mathcal{\Tilde{O}}(\max \{ \frac{L_{\max}}{\mu}n,\frac{L_{\max}}{\mu} n \sqrt{\frac{\alpha L}{\mu}} \})=\mathcal{\Tilde{O}}(\frac{L_{\max}}{\mu} n \sqrt{\frac{\alpha L}{\mu}})$ under the WGC. For $\mu$-PL objectives, these rates are worse than RR as the parameter $\rho$ in the SGC satisfies $\rho \geq 1$, and the parameter $\alpha $ in the WGC satisfies $\alpha \geq \frac{\mu}{L}$.   
\end{remark}
\subsection{Proof Sketch} In this section, we provide a proof sketch of Theorems \ref{thm:rr} and \ref{thm:ig} under the SGC. A similar approach is taken for the WGC. To start with, we first express the descent on the objective $f$ in terms of epochs. For a step size $\eta \leq \frac{1}{nL}$ \cite{nguyen2021unified}, we have the following
\begin{align}
    f(x^{t+1}_0) \leq f(x^{t}_0) - \frac{n\eta}{2}\lVert \nabla f(x^t_0) \rVert^2 + \frac{L_{\max}^2 \eta}{2} \sum_{i=0}^n \lVert x_i^t - x_0^t \rVert^2. \label{eq:main_prog}
\end{align}
The key is to upper bound the sum of deviations of each step  from its starting point within an epoch, i.e. $\sum_{i=0}^n \lVert x_i^t - x_0^t \rVert^2$. We want this term to be small to get a fast rate. To do so, we obtain the following results (see Lemma \ref{lem_ig} and Lemma \ref{lem_rr} in Appendix \ref{sec:usef_lemmas} by setting $\sigma = 0$) for IG and RR under the SGC.  
    \begin{align}
        \text{IG: } \qquad \sum_{i=0}^{n-1} \lVert x_i^t - x_0^t \rVert^2 &\leq 2 \eta^2 n^2 (n \rho) \lVert \nabla f(x_0^t) \rVert^2, \nn \\
        \text{RR: } \qquad \mathbb{E}[\sum_{i=0}^{n-1}\lVert x_i^t - x_0^t\rVert^2] &\leq 2 \eta^2 n^2 (\rho + n) \mathbb{E}[\lVert \nabla f(x_0^t) \rVert^2]. \nn
    \end{align}
Neglecting the effects of taking expectations, the difference in the deviation bound between RR and IG is reflected in the factors $\rho+n$ and $n \rho$. After substituting these results into \eqref{eq:main_prog}, we obtain
\begin{align}
    \text{IG: } \quad f(x_0^{t+1}) &\stackrel{(a)}{\leq} f(x_0^t) - \frac{n \eta}{2} (1-2L_{\max}^2\eta^2 n^2 \rho) \lVert \nabla f(x_0^t) \rVert^2 \label{eq:IG_SGC_f},\\
    \text{RR: } \quad \mathbb{E}[f(x_0^{t+1})] &\stackrel{(b)}{\leq} \mathbb{E}[f(x_0^t)]-\frac{n\eta}{4} (1- 4 L_{\max}^2 \eta^2 n \rho)\mathbb{E}[\lVert\nabla f(x_0^t) \rVert^2],\label{eq:RR_SGC_f}
\end{align}
 where both (a) and (b) require $\eta \sim \mathcal{\Tilde{O}}(\frac{1}{nL_{\max}})$. Note that for IG to make progress in decreasing $f$, we further require $\eta \sim \mathcal{\Tilde{O}}(\min \{ \frac{1}{nL_{\max}},  \frac{1}{L_{\max}n \sqrt{\rho}}\} = \frac{1}{L_{\max}n \sqrt{\rho}})$ as $\rho \geq 1$; whereas for RR, we require $\eta \sim \mathcal{\Tilde{O}}(\min \{ \frac{1}{nL_{\max}}, \frac{1}{L_{\max} \sqrt{n \rho}}\})$. Hence, the learning rate of RR can be larger than IG under the SGC.

\section{Experimental Results}
In this section, we aim to address the following questions with our experiments: \ding{172} How do the empirical convergence rates of RR, SGD, and IG depend on $n$ and $\rho$? \ding{173} Is a large number of epochs required for RR to outperform SGD under interpolation? Unless otherwise stated, we choose a constant learning rate in the range $[10^{-3}, 10^{-1}]$ that minimizes the train loss.

\subsection{Synthetic Experiments} \label{sec:syn}
We first conduct experiments on a binary classification synthetic dataset. Specifically, we use the code provided by Loizou et al.~\cite{loizou2021stochastic} to generate linearly-separable data with margin $\tau > 0$. For training, we use a squared hinge loss with or without $L_2$ regularization to simulate the convex and $\mu$-PL cases respectively. Note that squared hinge loss satisfies the SGC with $\rho\leq\frac{n}{\tau^2}$ (see Proposition \ref{sgc_more} in Appendix \ref{app:SGCWGC}). Thus by varying $\tau$ we can investigate the convergence of SGD, RR, and IG as a function of the SGC constant $\rho$. The results are shown in Figure \ref{fig:figure1}. In agreement with our bounds, we observe that: (i) RR outperforms SGD when $\rho \gg n$, (ii) RR outperforms IG for all values of $\rho$, and (iii) the convergence (naturally) slows down for increasing $\rho$ (which corresponds to decreasing margin $\tau$). 

\begin{figure*}[t]
    \centering
    \begin{subfigure}{0.32\textwidth}
        \includegraphics[width=\textwidth]{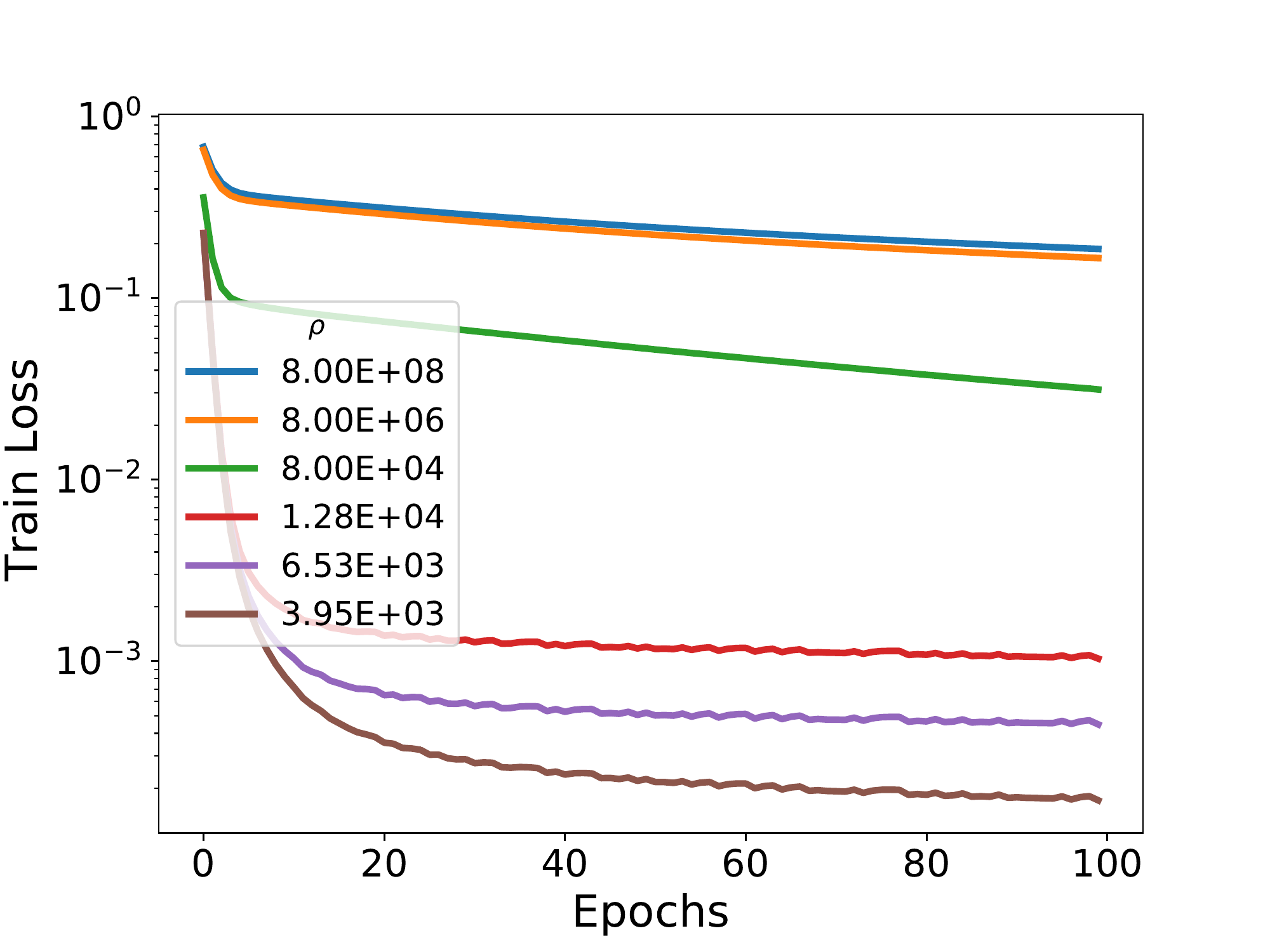}
        \caption{Train loss vs epochs}
        \label{fig:syn_a}
    \end{subfigure}
    \begin{subfigure}{0.32\textwidth}
        \includegraphics[width=\textwidth]{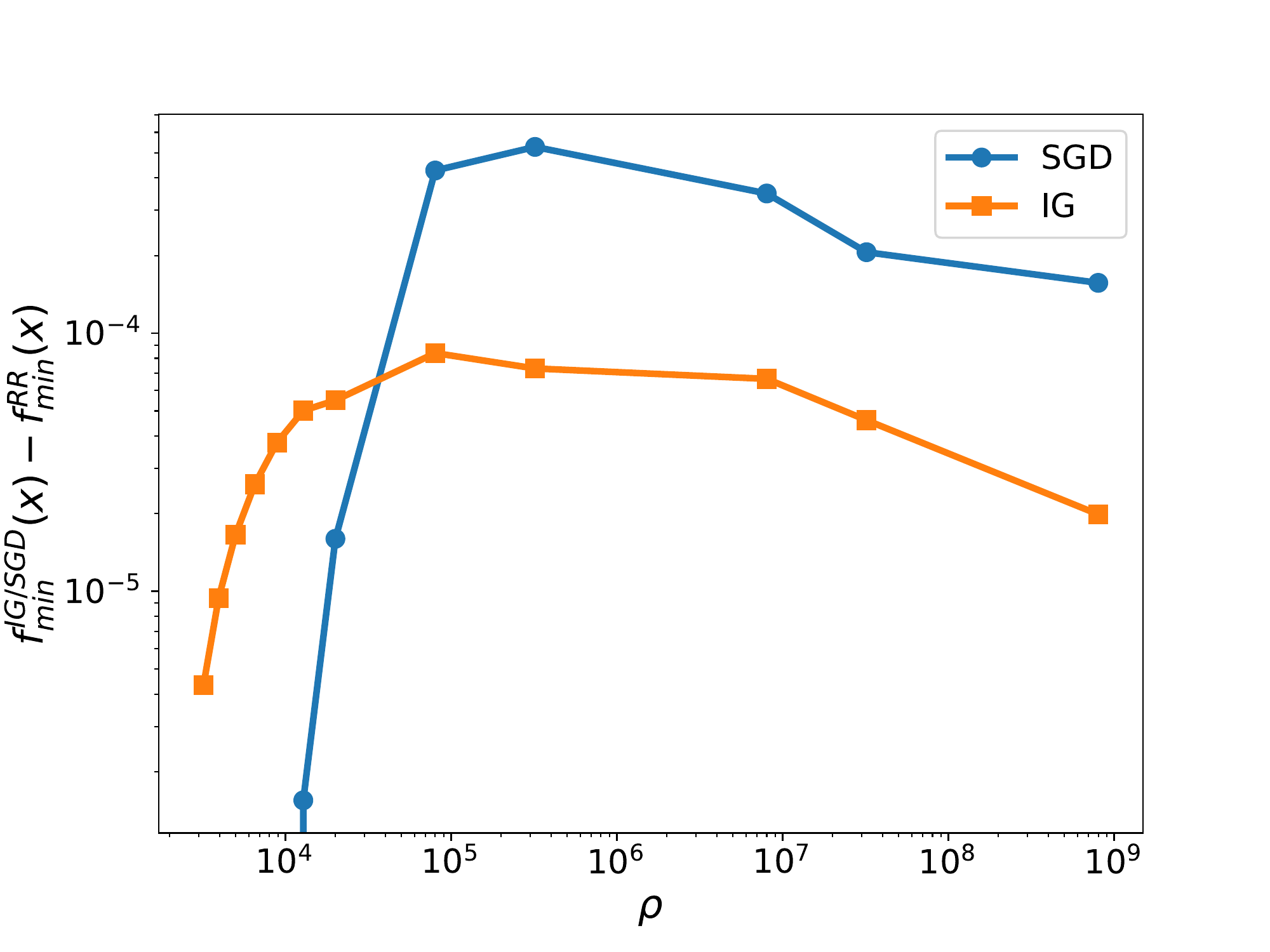}
        \caption{Squared hinge loss}
        \label{fig:syn_b}
    \end{subfigure}
    \begin{subfigure}{0.32\textwidth}
        \includegraphics[width=\textwidth]{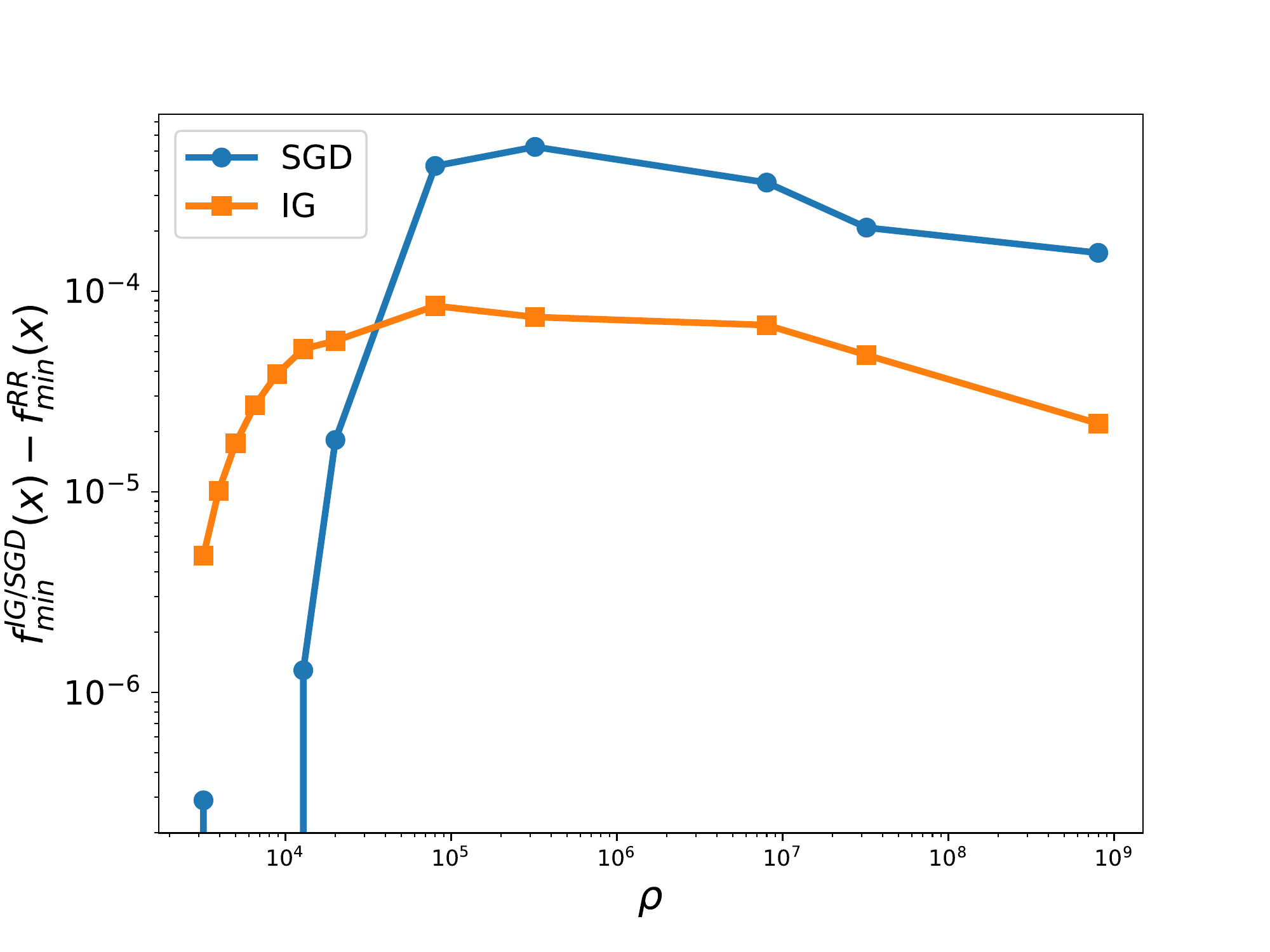}
        \caption{Squared hinge $L_2$ loss}
        \label{fig:syn_c}
    \end{subfigure}
    \caption{Binary classification on a linearly separable dataset with $n=800$. (a) We plot train loss of RR as a function of epochs using squared hinge loss for different values of $\rho$. We observe that the training is very slow when $\rho > \mathcal{O}(10^4)$, and speeds up when $\rho$ approaches the order of magnitude of $n$ or less. (b) We plot minimum train loss difference between RR and SGD or IG as a function of $\rho$. For a convex objective such as squared hinge loss, we observe that the performance gain of RR over SGD predominantly occurs when $\rho \gtrsim \mathcal{O}(10^4)$. For $\rho < \mathcal{O}(10^4)$, SGD achieves a lower train loss than RR. We also observe that RR consistently performs better than IG for all $\rho$ values. (c) Similar observations are made as those of (b) for a $\mu$-PL objective: squared hinge loss with $L_2$ regularization.}
    \label{fig:figure1}
\end{figure*}

\subsection{Binary classification using RBF kernels} \label{sec:rbf}
Besides understanding the influence of $\rho$, we further conduct experiments to study the convergence rates of RR in minimizing convex losses under interpolation. To this end, we use squared hinge loss, squared loss and logistic loss on the mushrooms dataset from LIBSVM using RBF kernels \cite{loizou2021stochastic}. We follow the procedure in Vaswani et al.~\cite{vaswani2019painless} for choosing the bandwidths of the RBF kernels. As shown by our bounds, a large epoch is not required for RR to outperform SGD. This is observed in Figure \ref{fig:figure3}, in which RR outperforms SGD in the first few epochs of training, and the performance gap is significantly larger in the early training stage than the later. 

\begin{figure*}[h!]
    \centering
    \begin{subfigure}{0.32\textwidth}
        \includegraphics[width=\textwidth]{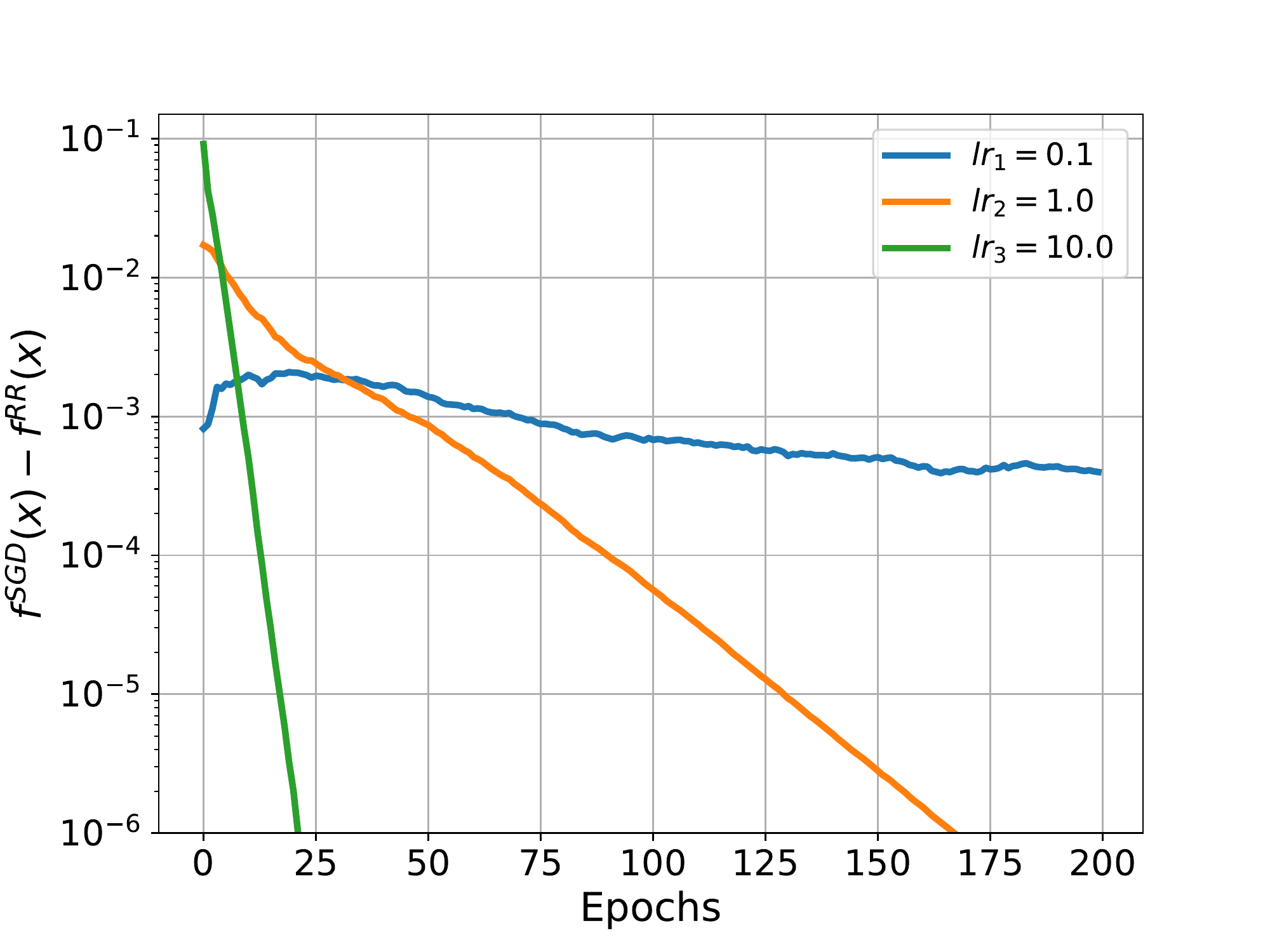}
        \caption{Squared hinge loss}
        \label{fig:rbf_a}
    \end{subfigure}
    \hfill
    \begin{subfigure}{0.32\textwidth}
        \includegraphics[width=\textwidth]{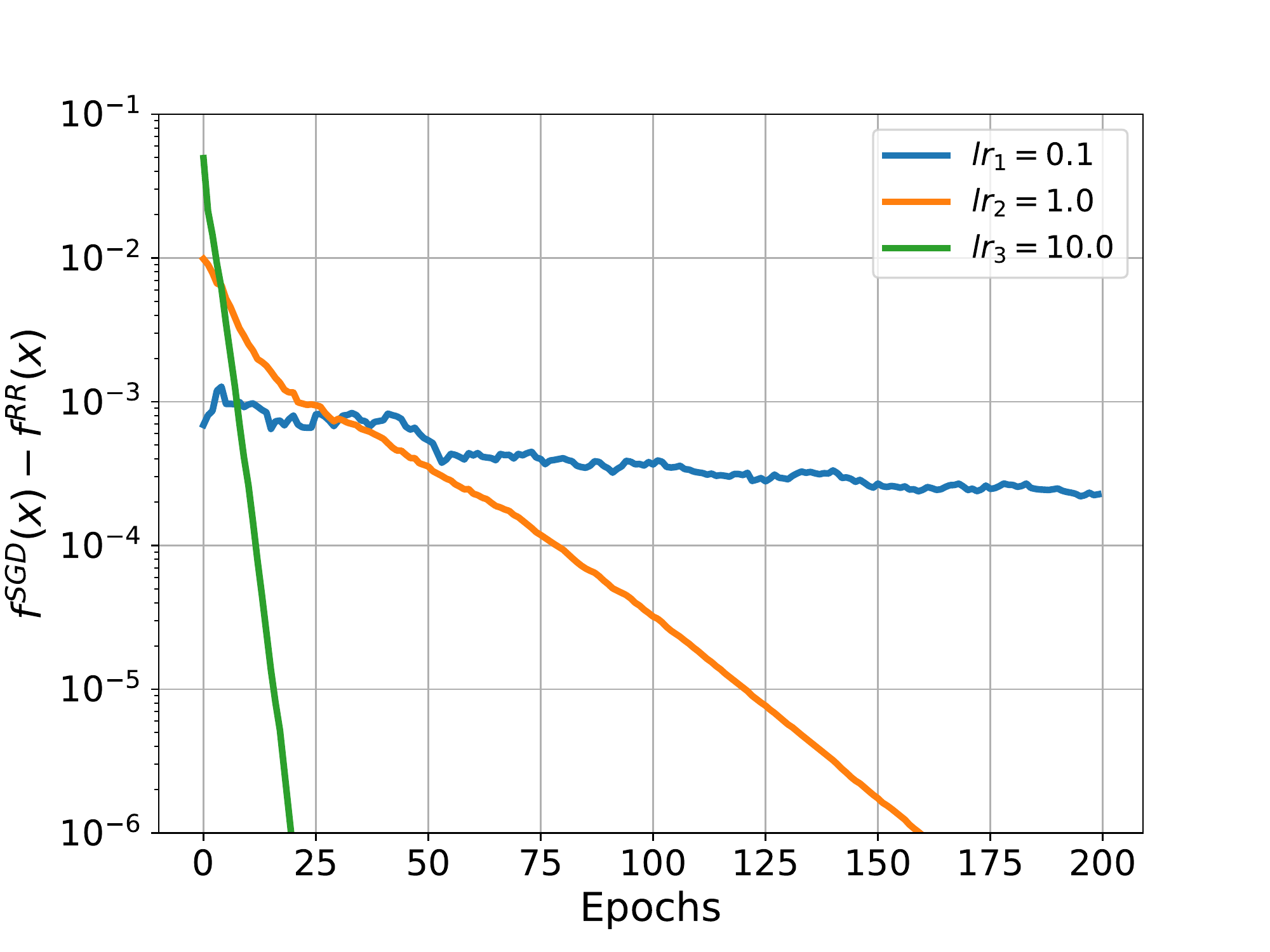}
        \caption{Squared loss}
        \label{fig:rbf_b}
    \end{subfigure}
    \begin{subfigure}{0.32\textwidth}
        \includegraphics[width=\textwidth]{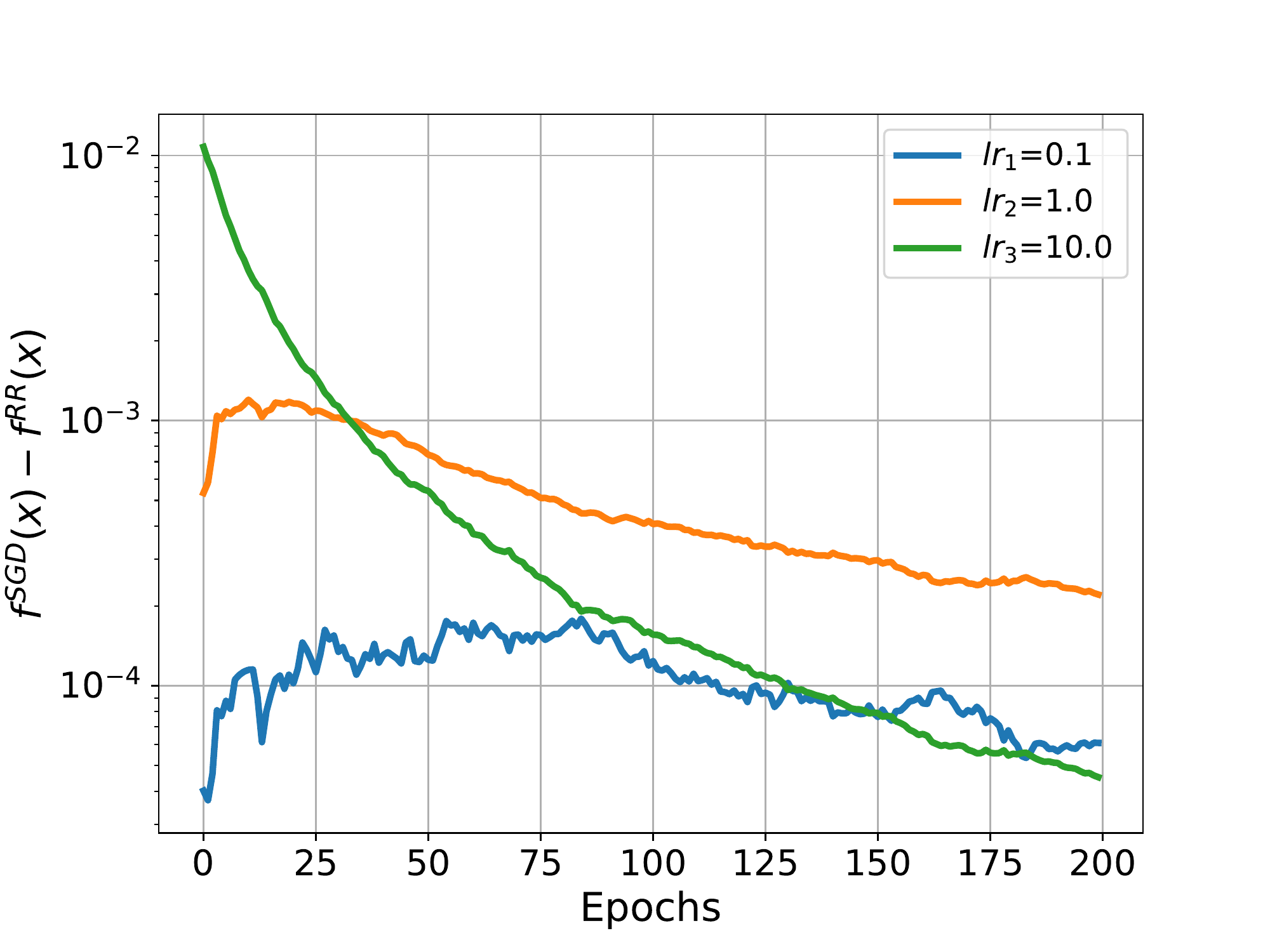}
        \caption{Logistic loss}
        \label{fig:rbf_c}
    \end{subfigure}
    \caption{Binary classfication on mushrooms dataset using RBF kernels. (a) We plot the difference in the squared loss between SGD and RR as a function of epochs for step sizes ($lr$) \textcolor{blue}{$0.1$}, \textcolor{orange}{$1.0$}, and \textcolor{green}{10.0}. (b)(c) Similar plot as (a) but for the squared loss and the logistic loss respectively. We observe that RR converges faster than SGD for all the step sizes, and the train loss difference is more significant in the initial training stage. Considering logistic loss and $lr = 10.0$, the train loss gap is $\tilde{\mathcal{O}}(10^{-2})$ after one epoch, and decreases to $\tilde{\mathcal{O}}(10^{-4})$ in the end. Similarly, we observe a faster decrease in the train loss difference within the first $25$ epochs of training when squared hinge loss or squared loss is used.}
    \label{fig:figure3}
\end{figure*}
\subsection{Multi-class classification using deep networks} \label{sec:mnist}
We conducted experiments to study the convergence rates of RR in non-convex settings under interpolation. To this end, we use a softmax loss and one hidden-layer multilayer perceptrons (1-MLP) of different widths on the MNIST image classification dataset \cite{lecun1998gradient}. The results in Figure \ref{fig:2} show that RR can converge faster than SGD for a small $T$, and a small $n$ favors RR while for large $n$ there is less of an advantage to using RR. 

The key take-aways from our experimental results is that for over-parameterized problems:
\begin{enumerate}
    \item RR consistently outperforms IG.
    \item The condition $n \ll \rho$ is crucial for RR to outperform SGD. 
    \item Under over-parametrization, RR can outperform SGD already at  the early stage of training (a large epoch is not required as suggested by previous theory for under-parameterized models).
\end{enumerate} 

\begin{figure}[h!]
	\centering     
    \centering
    \begin{subfigure}{0.32\textwidth}
        \includegraphics[width=\textwidth]{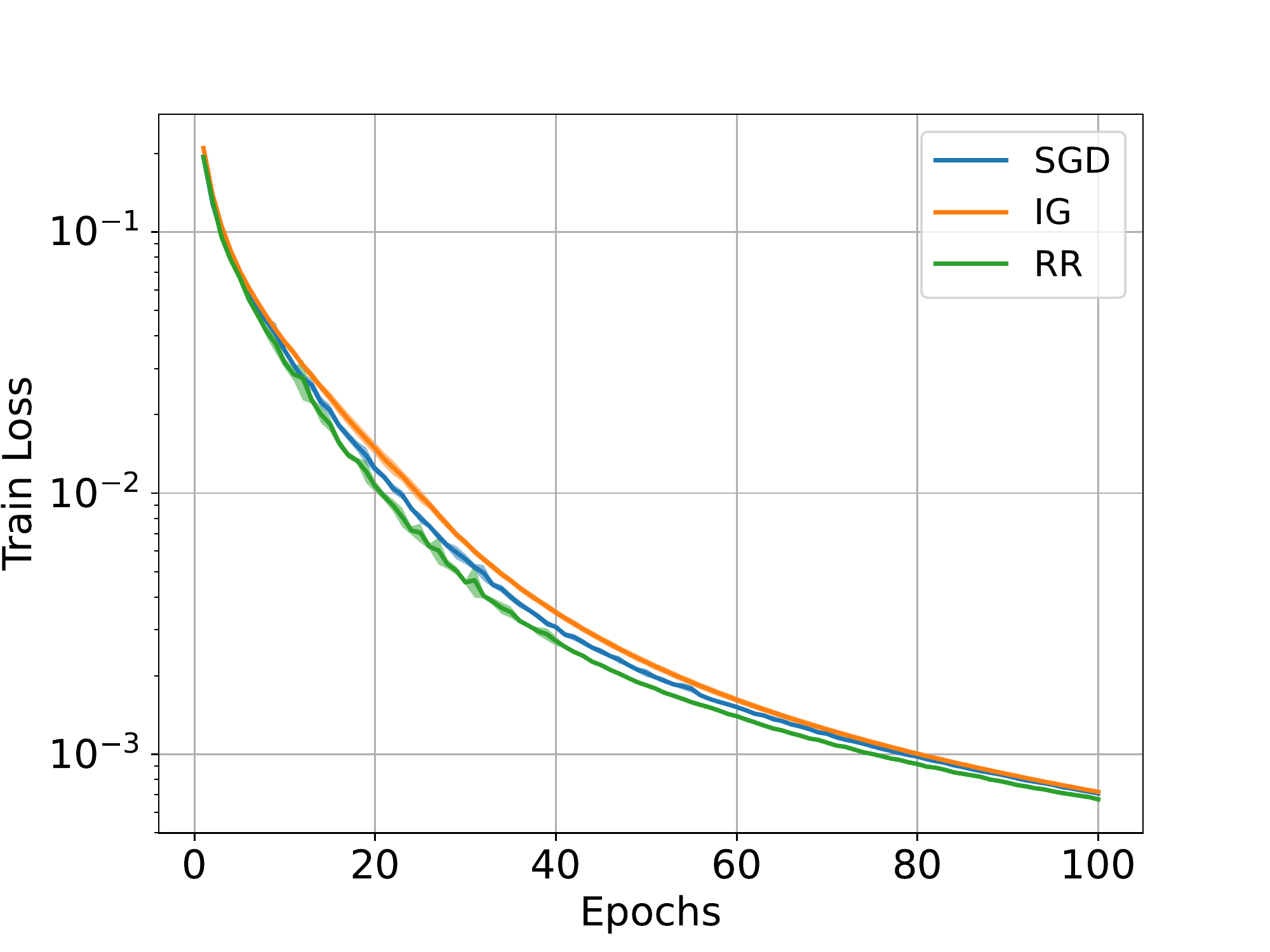}
        \caption{Train loss $(W=100)$}
        \label{fig:mnist_a}
    \end{subfigure}
    \begin{subfigure}{0.32\textwidth}
        \includegraphics[width=\textwidth]{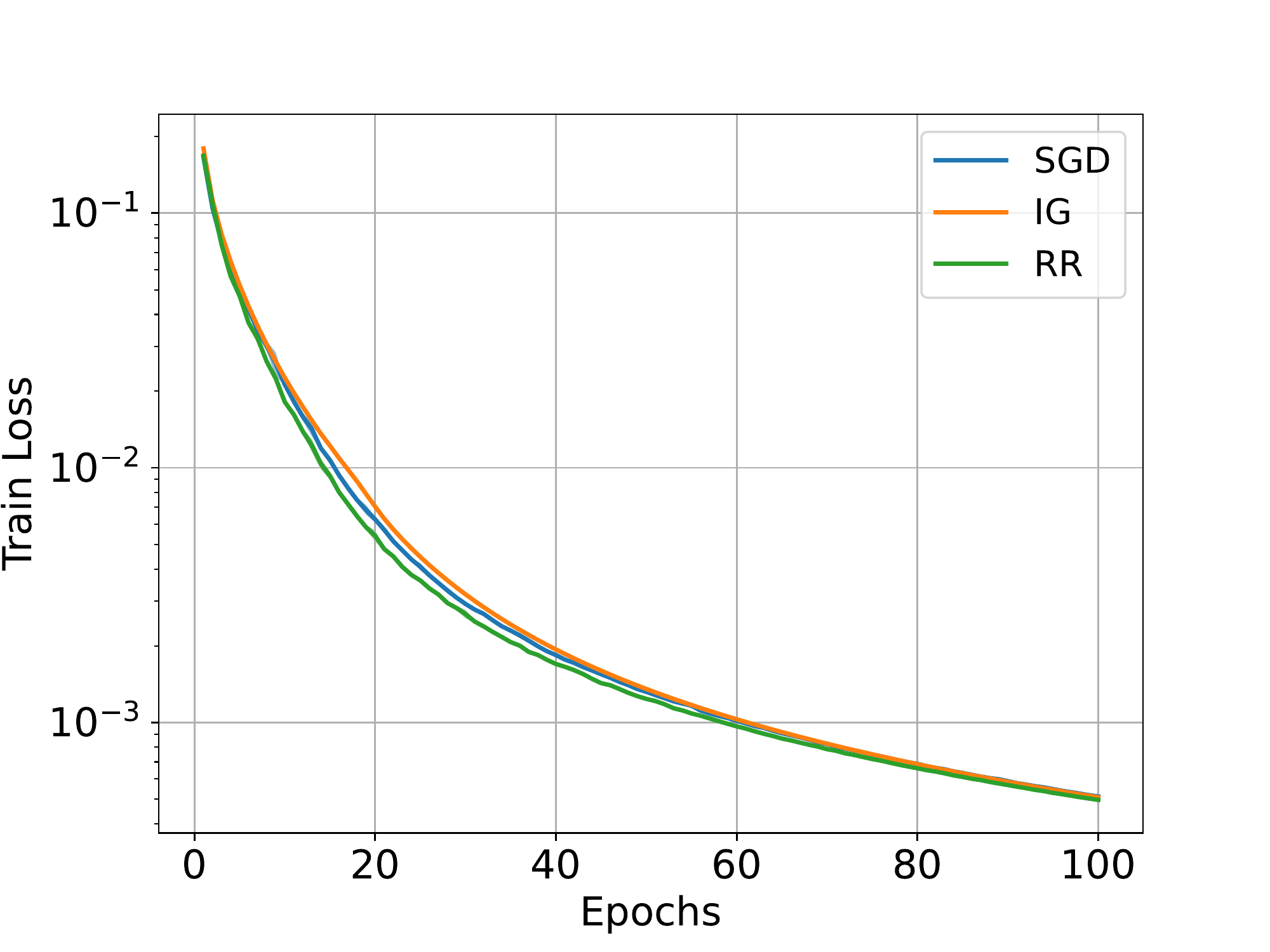}
        \caption{Train loss $(W=1000)$}
        \label{fig:mnist_b}
    \end{subfigure}
    \begin{subfigure}{0.32\textwidth}
        \includegraphics[width=\textwidth]{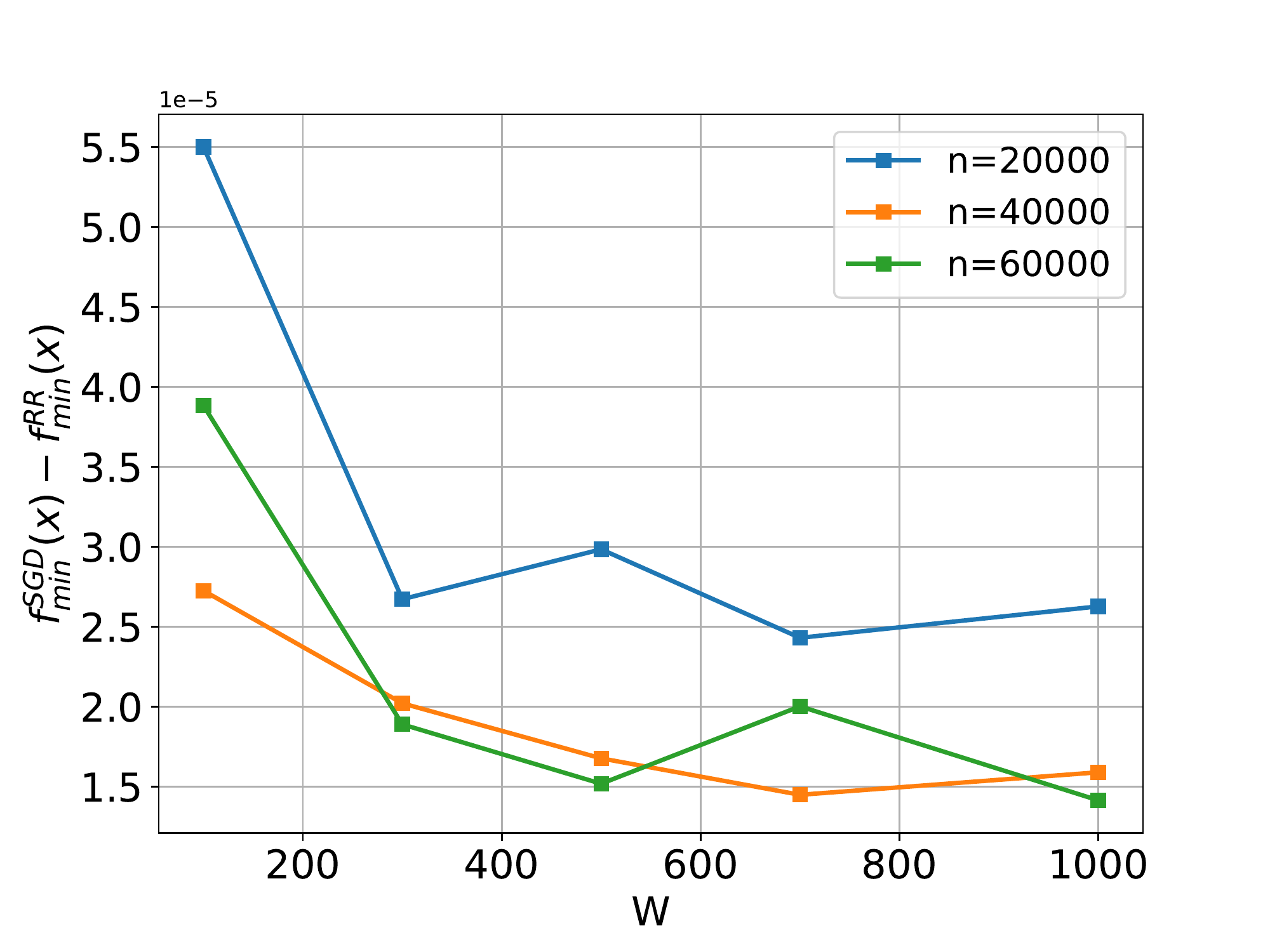}
        \caption{Different $n$s}
        \label{fig:mnist_c}
    \end{subfigure}
	\caption{MNIST digits classification with a softmax loss. (a) (b) We plot train loss as a function of epochs for 1-MLP of width ($W$) 100 and 1000 respectively. In both cases, we observe that RR converges faster than SGD in the initial training stage. (c) We plot minimum train loss difference between RR and SGD for $n=20000, 40000$ and $60000$ against different values of $W$. Despite RR outperforming SGD for all $W$s and $n$s, the performance gap is more significant when $n = 20000$.}
	\label{fig:2}
\end{figure}

\section{Conclusion}
In this paper, we have derived convergence rates of RR under interpolation as implied by the SGC or the WGC for $\mu$-PL objectives. In this setting, RR converge faster than SGD provided the key condition $\rho > n$ under the SGC or $\alpha\frac{L}{\mu} > n$ under the WGC holds (assuming the individual Lipschitz constants do not differ too much from the overall Lipschitz constant). 
Moreover, we show that RR outperforms IG for all values  of $\rho$ and $\alpha$ under the SGC and WGC respectively. 
Besides this, we further demonstrate that IG can outperform SGD when $\rho > n^2$ holds under the SGC or $\alpha\frac{L}{\mu} > n^2$ holds under the WGC. We remark that none of these conclusions follows from previous analysis under the strong convexity assumption. Our experimental results support these theoretical findings. 

\subsubsection{Ethical Statement} The contribution is the theoretical analysis of an existing algorithm, so it does not have direct societal or ethical implications.

\bibliographystyle{splncs04}
\bibliography{refs}
%





\clearpage
\appendix
\section{More Results for the SGC and WGC}
\label{app:SGCWGC}

The proposition below extends Vaswani et al.~\cite[Lemma~1]{vaswani2019fast} to a  function class that includes squared-hinge or logistic losses as special cases. 
\begin{proposition}\label{sgc_more}
    Assume binary linearly-separable dataset $(a_i,y_i), i\in[n]$ of size $n$ with margin $\tau:=\arg\max_{\|x\|_2=1}\min_{i\in[n]}y_ia_i^Tx>0$, normalized features $\|a_i\|_2\leq 1$ and $y_i\in\{\pm1\}$. Let $f(x;i)=\ell(y_ia_i^Tx)$ for a smooth monotonic function $\ell$. Then,  SGC \eqref{eq:SGC} holds with $\rho \leq {n}/{\tau^2}$. 
\end{proposition} 
\begin{proof}
We consider smooth monotonic non-increasing functions of the form $f_i(x) = l(y_i a_i^T x) = l(z_i^T x)$, where $z_i = y_i a_i$ and $a_i$ is the feature vector for the $i$th sample. We assume the values of $z_i$ are properly normalized such that $\max\limits_{i}\lVert z_i \rVert \leq 1$. Define $x^* = \argmax\limits_{\lVert x \rVert_2 = 1} \min\limits_i z_i^t x$ and $\tau = \max\limits_{\lVert x \rVert_2=1}\min\limits_i z_i^T x$. Then we have 
\begin{align}
    \lVert  \frac{1}{n} \sum_{i=1}^n \nabla f(x;i) \rVert^2_2 &=  \lVert \frac{1}{n} \sum_{i=1}^n l^{'}(z_i^T x) z_i \rVert_2^2.\\
\end{align}
Note that for a vector $u$, its $2$-norm is $\lVert u \rVert_2 = \max\limits_{\lVert v \rVert_2 = 1} v^T u$. Hence, we have the following 
\begin{align}
    \lVert  \frac{1}{n} \sum_{i=1}^n \nabla f(x;i) \rVert^2_2 &\geq \bigl( \frac{1}{n} \sum_{i=1}^n l^{'}(z_i^T x) z_i^T x^*\bigr)^2 \nonumber \\
    &= \tau^2 \bigl( \frac{1}{n} \sum_{i=1}^n l^{'}(z_i^T x)\bigr)^2 \nonumber \\
    &= \tau^2 \Bigl\{ \frac{1}{n^2} \sum_{i=1}^n l^{'}(z_i^T x)^2+\frac{1}{n^2} \sum_{i \neq j} l^{'}(z_i^T x) l^{'}(z_j^T x)\Bigr\} \nonumber \\
    &\geq \frac{\tau^2}{n} \Bigl( \frac{1}{n} \sum_{i=1}^n l^{'}(z_i^T x)^2\Bigr) \label{eq:prop1_line1}\\ 
    &\geq \frac{\tau^2}{n} \Bigl( \frac{1}{n} \sum_{i=1}^n l^{'}(z_i^T x)^2 \lVert z_i \rVert^2\Bigr) \nonumber\\
    &= \frac{\tau^2}{n} \frac{1}{n} \sum_{i=1}^n \lVert \nabla f(x;i)\rVert^2,
\end{align}
where the inequality in \eqref{eq:prop1_line1} follows because $l$ is monotonic, that is $l^{'}(t_1)l^{'}(t_2) \geq 0$ $\forall t_1,t_2 \in \mathbb{R}$. Therefore, we have the following 
\begin{align}
    \frac{1}{n}\sum_{i=1}^n \lVert \nabla f(x;i)\rVert^2 &\leq \frac{n}{\tau^2} \lVert  \frac{1}{n} \sum_{i=1}^n \nabla f(x;i) \rVert^2_2 \nonumber, 
\end{align}
hence we obtain $\rho \leq \frac{n}{\tau^2}$. 
\end{proof}

The next proposition was shown in prior work~\cite{mishkin2020interpolation} but was stated in a different way.
\begin{proposition}\label{prop:prop_sgc_wgc}
    Suppose $f$ satisfies Assumption \ref{smooth} and Assumption \ref{SGC} with some constant $\rho \geq 1$, then it also satisfies Assumption \ref{WGC} with $\alpha \leq \rho$. 
\end{proposition}
\begin{proof}
   $\frac{1}{n}\sum_{i=1}^n \lVert \nabla f(x;i)\rVert^2 \stackrel{(a)}{\le} \rho \lVert \nabla f(x)\rVert^2 \stackrel{(b)}{\le} 2 L \rho (f(x) - f(x^*))$, where (a) is based on Assumption \ref{SGC}, and (b) is based on a standard result of $L$-smooth function (see Garrigos and Gower \cite[Lem.~2.28]{garrigos2023handbook} for a proof). Therefore, the WGC is satisfied with $\alpha \leq \rho$.
\end{proof}

Below is a result that gives a lower bound on the parameter $\alpha$ for a smooth $\mu$-PL objective satifies the WGC. 
\begin{proposition}\label{prop:alpha_wgc} 
Suppose f satisfies Assumption \ref{def_pl}, \ref{smooth} and \ref{WGC}, then the parameter $\alpha$ in the WGC satisfies $\alpha \geq \frac{\mu}{L}$. 
\end{proposition}
\begin{proof}
    \begin{align}
        2\mu(f(x)-f(x^*)) \stackrel{(a)}{\leq} \lVert \nabla f(x) \rVert^2 \stackrel{(b)}{\leq} \frac{1}{n} \sum_{i=1}^n
 \lVert \nabla f(x;i) \rVert^2 \stackrel{(c)}{\leq} 2\alpha L(f(x)-f(x^*)),
 \end{align}
 where (a) is based on Assumption \ref{def_pl}; (b) is based on Jensen's inequality and the convexity of squared norms; (c) is based on Assumption \ref{WGC}. 
\end{proof}

\section{SGD Proofs under Interpolation} \label{sec:sgd_proofs}
\begin{theorem}[SGD for $\mu$-PL under WGC] \label{thm:sgd_wgc}
Suppose $f$ satisfies Assumption \ref{smooth} and \ref{WGC}, then SGD with a constant learning rate $\eta = \frac{\mu}{L^2 \alpha}$ has the following convergence rate for $\mu$-PL objectives:
\begin{align}
    \mathbb{E}[f(x^{K+1}) - f(x^*)] \leq (1-\frac{\mu^2}{L^2 \alpha})^K (f(x^0)-f(x^*)),  
\end{align}
where $K$ is the total number of iterations.
\end{theorem}
\begin{proof}
 Take expectation conditioned on $x^k$
 \begin{align}
     \mathbb{E}_k[f(x^{k+1})] &\stackrel{(a)}{\leq} f(x^k) - \langle \nabla f(x^k), \mathbb{E}_k[x^{k+1} - x^k]\rangle + \frac{L}{2} \mathbb{E}_k[\lVert x^{k+1} - x^k\rVert^2] \nn \\
     &\stackrel{(b)}{=} f(x^k) - \eta_k \lVert \nabla f(x^k) \rVert^2 + \frac{L \eta_k^2}{2} \mathbb{E}_k[\lVert \nabla f(x^k;i_k)\rVert^2] \nn \\
    &\stackrel{(c)}{\leq} f(x^k) - \eta_k 2 \mu (f(x^k) - f(x^*)) + L^2 \eta_k^2 \alpha (f(x^k) - f(x^*)),
 \end{align}
 where (a) is based Assumption \ref{smooth}; (b) is by $\mathbb{E}[\nabla f(x^k;i_k)] = \nabla f(x^k)$; (c) is based on Assumption \ref{WGC} and \ref{def_pl}. Take total expectation and subtract $f(x^*)$ on both sides
 \begin{align}
     \mathbb{E}[f(x^{k+1})-f(x^*)] &\leq [1-2\mu \eta_k(1-\frac{L^2\eta_k \alpha}{2\mu})] \mathbb{E}[f(x^{k})-f(x^*)] \nn \\
     &\stackrel{(d)}{=} (1-\frac{\mu^2}{L^2 \alpha}) \mathbb{E}[f(x^k)-f(x^*)],
 \end{align}
 where (d) is obtained by substituting $\eta_k = \eta =\frac{\mu}{L^2 \alpha}$. Then solve the relationship recursively.
\end{proof}

\section{Useful Lemmas} \label{sec:usef_lemmas}
We draw ideas from Nguyen et al.~\cite{nguyen2021unified} and Mishchenko et al.~\cite{mishchenko2020random} to construct our proofs. The high-level idea is to first bound the decrease in the objective value (see Lemma \ref{lem_descent}), then bound the progression term $\sum_{i=0}^{n-1}\lVert x_i^t - x_0^t\rVert^2$ for IG and  $\mathbb{E}[\sum_{i=0}^{n-1}\lVert x_i^t - x_0^t\rVert^2]$ for RR respectively (see Lemma \ref{lem_ig} and Lemma \ref{lem_rr}).
In this section, we present these lemmas that will be used in the theory proofs. 
Lemma \ref{lem_exp} is a restatement of Mishchenk et al.~\cite[Lemma 1]{mishchenko2020random}.

\begin{lemma} \label{lem_exp}
Let $X_1,...,X_n$ be a given set of vectors in $\mathbb{R}^d$, denote their average to be $\bar{X} = \frac{1}{n}\sum_{i=1}^n X_i$ and population variance to be $\sigma^2 = \frac{1}{n}\sum_{i=1}^n\lVert X_i - \bar{X}\rVert^2$. Fix k $\in\{1,...,n\}$, let $X_{\pi_1},...,X_{\pi_k}$ be sampled uniformly without replacement from $\{X_1,...,X_n\}$ and $\bar{X}_{\pi}$ be their average. Then the following hold true 
\begin{align}
    \mathbb{E}[\bar{X}_{\pi}] = \bar{X} \quad \quad \quad \mathbb{E}[\lVert \bar{X}_\pi - \bar{X} \rVert^2] = \frac{n-k}{k(n-1)}\sigma^2.    
\end{align}
\end{lemma}

Lemma \ref{lem_descent} is proved in Nguyen et al.~\cite[Lemma 8]{nguyen2021unified}. Here we replace $L$ with $L_{\max}$ and take the step size $\eta$ to be constant for all iterations and epochs. 

\begin{lemma}\label{lem_descent}
Suppose Assumption \ref{smooth} holds. Given a shuffling scheme $\{\pi^t\}_t$ and a constant step size $\eta$ such that $ \eta \leq \frac{1}{nL}$, we have the following:
\begin{align}
    f(x^{t+1}_0) \leq f(x^{t}_0) - \frac{n\eta}{2}\lVert \nabla f(x^t_0) \rVert^2 + \frac{L_{\max}^2 \eta}{2} \sum_{i=0}^n \lVert x_i^t - x_0^t \rVert^2. \label{eq:lem_descent_ig}
\end{align}
Taking total expectation over the randomness
\begin{align}
    \mathbb{E} [f(x^{t+1}_0)] \leq \mathbb{E} [f(x^{t}_0)] - \frac{n\eta}{2} \mathbb{E} [\lVert \nabla f(x^t_0) \rVert^2] + \frac{L_{\max}^2 \eta}{2} \sum_{i=0}^n \mathbb{E} [\lVert x_i^t - x_0^t \rVert^2]. \label{eq:lem_descent_rr}
\end{align}
\end{lemma}
The remaining Lemmas \ref{lem_ig}-\ref{lem_rr} concern the update rules in Eqn \eqref{eq:update2}. Recall that $x_0^{t+1} = x_n^t$ $\forall t \geq 0$. Also recall that for RR the permutation vector $\pi^t$ is chosen randomly at each epoch $t$ while it is fixed for IG.
Lemma \ref{lem_ig} will be combined with Lemma \ref{lem_descent} in the proofs of IG. Lemma \ref{lem_rr} will be combined with Lemma \ref{lem_descent}  in the proofs of RR. 
\begin{lemma}\label{lem_ig}
Suppose Assumptions \ref{smooth} hold. For a step size $\eta \leq \frac{1}{\sqrt{2}n L_{\max}}$, it holds for IG that
under Assumption \ref{SGC_noise},
\begin{align}
    \sum_{i=0}^{n-1} \lVert x_i^t - x_0^t \rVert^2 \leq 2\eta^2 n^3 \rho \lVert \nabla f(x_0^t)\rVert^2 + 2\eta^2 n^3 \sigma^2, \label{eq:lem_ig_sgc}
\end{align}
and under Assumption \ref{WGC_noise},
\begin{align}
    \sum_{i=0}^{n-1} \lVert x_i^t - x_0^t \rVert^2 \leq 4 \eta^2 n^3 \alpha L (f(x_0^t) - f(x^*)) + 2\eta^2 n^3 \sigma^2. \label{eq:lem_ig_wgc}
\end{align}
\end{lemma}
\begin{proof} 
By the update rule
\begin{align}
    \lVert x_i^t - x_0^t \rVert^2 &= \eta^2 \lVert \sum_{j=0}^{i-1} \nabla f(x_j^t;j+1) \rVert^2 \nonumber \\
    &= \eta^2 \lVert \sum_{j=0}^{i-1} \nabla f(x_j^t;j+1) - \nabla f(x_0^t;j+1) + \nabla f(x_0^t;j+1) \rVert^2 \nonumber \\
    &\stackrel{(a)}{\leq}  2\eta^2 \lVert \sum_{j=0}^{i-1} \nabla f(x_j^t;j+1) - \nabla f(x^t_0;j+1)\rVert^2 + 2\eta^2 \lVert \sum_{j=0}^{i-1} \nabla f(x_0^t;j+1) \rVert^2 \nonumber\\
     &\stackrel{(b)}{\leq} 2 \eta^2 i\sum_{j=0}^{i-1} \lVert \nabla f(x_j^t;j+1) - \nabla f(x_0^t;j+1) \rVert^2 + 2 \eta^2 i \sum_{j=0}^{i-1}\lVert \nabla f(x_0^t;j+1) \rVert^2 \nonumber\\
    &\stackrel{(c)}{\leq} 2 \eta^2 L_{\max}^2 i \sum_{j=0}^{i-1} \lVert x_j^t - x_0^t \rVert^2 + 2 \eta^2 i \sum_{j=0}^{i-1}\lVert \nabla f(x_0^t;j+1) \rVert^2, 
\end{align}
where in (a) and (b) are based on Jensen's inequality and convexity of squared norms, and (c) is based on Assumption \ref{smooth}. Summing over $i=0,...,n-1$ gives
\begin{align}
    \sum_{i=0}^{n-1} \lVert x_i^t - x_0^t \rVert^2 &\leq 2 \eta^2 L_{\max}^2 \sum_{i=0}^{n-1} i \sum_{j=0}^{i-1} \lVert x_j^t - x_0^t \rVert^2  + 2 \eta^2 \sum_{i=0}^{n-1} i \sum_{j=0}^{i-1} \lVert \nabla f(x_0^t;j+1) \rVert^2  \nonumber \\
     &\leq 2 \eta^2 L_{\max}^2 \sum_{i=0}^{n-1} i \sum_{j=0}^{n-1} \lVert x_j^t - x_0^t \rVert^2  + 2 \eta^2 \sum_{i=0}^{n-1} i \sum_{j=0}^{n-1} \lVert \nabla f(x_0^t;j+1) \rVert^2  \nonumber \\
    &\stackrel{(d)}{\leq} \eta^2 L_{\max}^2 n^2 \sum_{i=0}^{n-1} \lVert x_i^t - x_0^t \rVert^2 + \eta^2 n^3 [\frac{1}{n} \sum_{i=0}^{n-1} \lVert \nabla f(x_0^t;i+1) \rVert^2], 
\end{align}
where in (d) we have used $\sum_{i=0}^{n-1} i \leq \frac{n^2}{2}$. Rearrange we obtain the following: 
\begin{align}
    \sum_{i=0}^{n-1} \lVert x_i^t - x_0^t \rVert^2 &\leq \frac{\eta^2 n^3}{1- \eta^2 L_{\max}^2 n^2} [\frac{1}{n} \sum_{i=0}^{n-1} \lVert \nabla f(x_0^t;i+1) \rVert^2]. 
\end{align}
Choose a learning rate $\eta \leq \frac{1}{\sqrt{2} n L_{\max}}$ leads to
\begin{align}
      \sum_{i=0}^{n-1} \lVert x_i^t - x_0^t \rVert^2 &\leq 2 \eta^2 n^3 [\frac{1}{n} \sum_{i=0}^{n-1} \lVert \nabla f(x_0^t;i+1) \rVert^2].  
\end{align}
If Assumption \ref{SGC_noise} holds, then we obtain
\begin{align}
    \sum_{i=0}^{n-1} \lVert x_i^t - x_0^t \rVert^2 \leq 2 \eta^2 n^3 \rho \lVert \nabla f(x_0^t) \rVert^2 + 2 \eta^2 n^3 \sigma^2. 
\end{align}
If Assumption \ref{WGC_noise} holds, then we obtain
\begin{align}
     \sum_{i=0}^{n-1} \lVert x_i^t - x_0^t \rVert^2 \leq 4 \eta^2 n^3 \alpha L (f(x_0^t) - f(x^*)) + 2 \eta^2 n^3 \sigma^2. 
\end{align}
\end{proof}

\begin{lemma}\label{lem_rr}
Suppose Assumption \ref{smooth} hold. For a step size $\eta \leq \frac{1}{\sqrt{3} n L_{\max}}$, it holds for RR that
under Assumption \ref{SGC_noise},
\begin{align}
    \mathbb{E}\Bigl[\sum_{i=0}^{n-1}\lVert x_i^t - x_0^t \rVert^2\Bigr] \leq 2 n^2 \eta^2 (\rho+n) \mathbb{E}[\lVert \nabla f(x_0^t) \rVert^2] +  2 \eta^2 n^2 \sigma^2, \label{eq:lem_rr_sgc}
\end{align}
and under Assumption \ref{WGC_noise},
\begin{align}
    \mathbb{E}[\sum_{i=0}^{n-1}\lVert x_i^t - x_0^t\rVert^2]\leq 4\eta^2 n^2 \alpha L \mathbb{E}[f(x_0^t) - f(x^*)]+ 2 \eta^2 n^3 \mathbb{E}[\lVert \nabla f(x_0^t) \rVert^2] +  2 \eta^2 n^2 \sigma^2. \label{eq:lem_rr_wgc}
\end{align}
\end{lemma}
\begin{proof} 
By the update rule
\begin{align}
    \lVert x_i^t - x_0^t \rVert^2  &= \eta^2 \lVert \sum_{j=0}^{i-1} \nabla f(x_j^t;\pi_{j+1}^t)\rVert^2 \nonumber\\ 
    &= \eta^2 i^2 \lVert \frac{1}{i} \sum_{j=0}^{i-1}[\nabla f(x_j^t;\pi^{t}_{j+1}) - \nabla f(x_0^t; \pi_{j+1}^t) + \nabla f(x_0^t;\pi_{j+1}^t) - \nabla f(x_0^t) + \nabla f(x_0^t)] \rVert^2 \nonumber\\
    &\stackrel{(a)}{\leq} 3 \eta^2 i^2 \lVert \frac{1}{i} \sum_{j=0}^{i-1} [\nabla f(x_j^t;\pi^t_{j+1})-\nabla f(x_0^t;\pi_{j+1}^t)]\rVert^2 + 3\eta^2 i^2 \lVert \frac{1}{i} \sum_{j=0}^{i-1}[\nabla f(x_0^t;\pi^t_{j+1}) \nonumber\\ 
    &\quad -\nabla f(x_0^t)]\rVert^2 + 3 \eta^2 i^2 \lVert \nabla f(x_0^t)\rVert^2 \nonumber \nonumber\\
    &\stackrel{(b)}{\leq} 3 \eta^2 i \sum_{j=0}^{i-1} \lVert \nabla f(x_j^t;\pi^t_{j+1})-\nabla f(x_0^t;\pi_{j+1}^t)\rVert^2 + 3\eta^2 i^2 \lVert \frac{1}{i} \sum_{j=0}^{i-1}[\nabla f(x_0^t;\pi^t_{j+1}) \nonumber\\ 
    &\quad -\nabla f(x_0^t)]\rVert^2 + 3 \eta^2 i^2 \lVert \nabla f(x_0^t)\rVert^2 \nonumber \nonumber\\
    &\stackrel{(c)}{\leq} 3\eta^2 i L_{\max}^2\sum_{j=0}^{i-1} \lVert x_j^t-x_0^t\rVert^2 + 3\eta^2 i^2 \lVert \frac{1}{i} \sum_{j=0}^{i-1} \nabla f(x_0^t;\pi^t_{j+1}) -\nabla f(x_0^t)\rVert^2 \label{eq:lem_RR_1} \\
    & \quad + 3 \eta^2 i^2 \lVert \nabla f(x_0^t)\rVert^2,     \nonumber
\end{align}
where (a) and (b) are based on Jensen's inequality and the convexity of squared norms; (c) is based on Assumption \ref{smooth}. Let $\sigma ^t$ be a sigma algebra on the iterates $\{ x_0^1,...,x_0^t \}$. We take conditional expectation w.r.t $\sigma^t$ and apply Lemma \ref{lem_exp} to find that 
\begin{align}
    \mathbb{E}\Bigl[\lVert\frac{1}{i} \sum_{j=0}^{i-1} & \nabla f(x_0^t;\pi^t_{j+1})  -\nabla f(x_0^t)\rVert^2 | \sigma^t \Bigr] = \frac{n-i}{i(n-1)} \frac{1}{n} \sum_{j=0}^{n-1} \lVert \nabla f(x_0^t;j+1) - \nabla f(x_0^t) \rVert^2 \nonumber \nn\\ 
    & = \frac{n-i}{i(n-1)} \bigl[ \frac{1}{n} \sum_{j=0}^{n-1} [\lVert \nabla f(x_0^t;j+1)\rVert^2 + \lVert \nabla f(x_0^t) \rVert^2 - 2 \langle \nabla f(x_0^t;j+1), \nabla f(x_0^t)\rangle] \bigr] \nn\\
    &  = \frac{n-i}{i(n-1)} \bigl[\frac{1}{n} \sum_{j=0}^{n-1} \lVert \nabla f(x_0^t;j+1)\rVert^2 - \lVert \nabla f(x_0^t)\rVert^2 \bigl] \nn \\
    & \leq \frac{n-i}{i(n-1)} \frac{1}{n} \sum_{j=0}^{n-1} \lVert \nabla f(x_0^t;j+1)\rVert^2. \label{eq:lem_RR_2}
\end{align}
Take conditional expectation of \eqref{eq:lem_RR_1} and substitute \eqref{eq:lem_RR_2} back
\begin{align}
    \mathbb{E}\Bigl[ \lVert x_i^t - x_0^t\rVert|\sigma^t \Bigr] &\leq 3 \eta^2 i L_{\max}^2\sum_{j=0}^{n-1} \mathbb{E} [\lVert x_j^t - x_0^t\rVert | \sigma^t] + 3 \eta^2 \frac{i(n-i)}{(n-1)}\frac{1}{n} \sum_{j=0}^{n-1} \lVert \nabla f(x_0^t;j+1)\rVert^2 + \nonumber\\
    & \quad + 3\eta^2 i^2 \lVert \nabla f(x_0^t) \rVert^2. \label{eq:lem_RR_e} 
\end{align}
Next, take total expectation and sum over $i=0,...,n-1$: 
\begin{align}
    \mathbb{E}\Bigl[ \sum_{i=0}^{n-1}\lVert x_i^t - x_0^t\rVert^2 \Bigr] &\leq 3\eta^2 L_{\max}^2 \mathbb{E}[\sum_{j=0}^{n-1}\lVert x_j^t-x_0^t\rVert^2]\sum_{i=0}^{n-1} i +  3\eta^2\frac{1}{n-1} \mathbb{E} [\frac{1}{n} \sum_{j=0}^{n-1} \lVert \nabla f(x_0^t;j+1)\rVert^2] \sum_{i=0}^{n-1}i(n-i) + \nonumber\\
    & \quad + 3\eta^2\mathbb{E}[\lVert \nabla f(x_0^t) \rVert^2]\sum_{i=0}^{n-1} i^2 \nonumber \label{eq:lem_RR_3}\\
    & \stackrel{(d)}{\leq} \frac{3}{2}\eta^2 n^2 L_{\max}^2 \mathbb{E}[\sum_{j=0}^{n-1}\lVert x_j^t-x_0^t\rVert^2] + \eta^2 n^2 \mathbb{E} [\frac{1}{n} \sum_{j=0}^{n-1} \lVert \nabla f(x_0^t;j+1)\rVert^2] + \nonumber\\
    & \quad + \eta^2 n^3 \mathbb{E}[\lVert \nabla f(x_0^t) \rVert^2]
\end{align}
In (d), we have used $\sum_{i=0}^{n-1}i \leq \frac{n^2}{2}$, $\sum_{i=0}^{n-1}i^2 \leq \frac{n^3}{3}$, and $\sum_{i=0}^{n-1}i(n-i) \leq \frac{n^2(n-1)}{3}$. Choosing $\eta \leq \frac{1}{\sqrt{3}L_{\max} n}$, we have 
\begin{align}
    \mathbb{E}[\sum_{i=0}^{n-1}\lVert x_i^t - x_0^t\rVert^2]\leq 2 \eta^2 n^2 \mathbb{E} [\frac{1}{n} \sum_{i=0}^{n-1} \lVert \nabla f(x_0^t;j+1)\rVert^2] + 2 \eta^2 n^3 \mathbb{E}[\lVert \nabla f(x_0^t) \rVert^2] 
\end{align}
If Assumption \ref{SGC_noise} holds, then we obtain 
\begin{align}
        \mathbb{E}[\sum_{i=0}^{n-1}\lVert x_i^t - x_0^t\rVert^2]\leq 2 \eta^2 n^2 (\rho + n) \mathbb{E}[\lVert \nabla f(x_0^t) \rVert^2] + 2 \eta^2 n^2 \sigma^2. 
\end{align}
If Assumption \ref{WGC_noise} holds, then we obtain
\begin{align}
        \mathbb{E}[\sum_{i=0}^{n-1}\lVert x_i^t - x_0^t\rVert^2] &\leq 4\eta^2 n^2 \alpha L \mathbb{E}[f(x_0^t) - f(x^*)]+ 2 \eta^2 n^3 \mathbb{E}[\lVert \nabla f(x_0^t) \rVert^2] +  2 \eta^2 n^2 \sigma^2.  
\end{align}
\end{proof}
\begin{lemma}\label{lem_recur}
    Consider the following recurrence satisfied for some $\eta \leq \min \{ \frac{1}{\eta_1}, \frac{1}{\eta_2}\}$ 
    \begin{align}
        \delta_T \leq (1-\eta \mu n)^T \delta_0, \label{eq:recur_form}
    \end{align}
    where $T$ is the total number of epochs, then the sample complexity is
    \begin{align}
        \mathcal{O} (\max \{ \eta_1, \eta_2 \} \frac{1}{\mu} log(\frac{\delta_0}{\epsilon}))
    \end{align}
\end{lemma}
\begin{proof}
We want 
    \begin{align}
        \delta_T \leq exp(-\eta \mu n T) \delta_0 \stackrel{(a)}{=} exp(-\eta \mu K) \delta_0 < \epsilon.
    \end{align}
    Note that (a) holds because $T = \frac{K}{n}$, where $K$ is the total number of iterations. Thus, we have 
    \begin{align}
        K \geq \frac{1}{\eta} \frac{1}{\mu} log(\frac{\delta_0} {\epsilon}) \geq \max \{ \eta_1, \eta_2\} \frac{1}{\mu} log(\frac{\delta_0}{\epsilon}). 
    \end{align}
\end{proof}
\section{Main Theorem Proofs} \label{sec:main_proof}
\subsection{Theorem \ref{thm:rr} Proof}
\begin{proof}
    We first show the proof under Assumption \ref{SGC_noise}. Substitute \eqref{eq:lem_rr_sgc} from Lemma \ref{lem_rr} into \eqref{eq:lem_descent_rr} in Lemma \ref{lem_descent} 
    \begin{align}
        \mathbb{E}[f(x_0^{t+1})] &\leq E[f(x_0^{t})] - \frac{n\eta}{2}\mathbb{E}[\lVert \nabla f(x_0^t)\rVert^2] + \frac{L_{\max}^2\eta}{2}(2 n^2 \eta^2 (\rho+n) \mathbb{E}[\lVert \nabla f(x_0^t) \rVert^2] +  2 \eta^2 n^2 \sigma^2)
        \nn \\
        & = E[f(x_0^{t})] - \frac{n \eta}{2}(1-2L_{\max}^2n^2\eta^2)\mathbb{E}[\lVert \nabla f(x_0^t)\rVert^2] + L_{\max}^2 \eta^3 n^2 \rho \mathbb{E}[\lVert \nabla   f(x_0^t)\rVert^2] +  L_{\max}^2 \eta^3 n^2 \sigma^2 \nn \\ 
        &\stackrel{(a)}{\leq} E[f(x_0^{t})] - \frac{n\eta}{4} \mathbb{E}[\lVert \nabla f(x_0^t)\rVert^2] + L_{\max}^2 \eta^3 n^2 \rho \mathbb{E}[\lVert \nabla   f(x_0^t)\rVert^2] +  L_{\max}^2 \eta^3 n^2 \sigma^2 \nn \\
        & = E[f(x_0^{t})] - \frac{n \eta}{4}(1-4L_{\max}^2
        \eta^2 n \rho)\mathbb{E}[\lVert \nabla f(x_0^t)\rVert^2] + L_{\max}^2 \eta^3 n^2 \sigma^2 \nn \\
        &\stackrel{(b)}{\leq}  E[f(x_0^{t})] - \frac{n \eta}{8}\mathbb{E}[\lVert \nabla f(x_0^t)\rVert^2] +L_{\max}^2 \eta^3 n^2 \sigma^2 \nn \\
        &\stackrel{(c)}{\leq} E[f(x_0^{t})] - \frac{n \mu \eta}{4}\mathbb{E}[f(x_0^t)-f(x^*)] + L_{\max}^2 \eta^3 n^2 \sigma^2,
    \end{align}
    where in (a) we choose $\eta \leq \frac{1}{2L_{\max}n}$, in (b) we choose $\eta \leq \frac{1}{2\sqrt{2} L_{\max}\sqrt{n \rho}}$, and (c) follows from Assumption \ref{def_pl}. Subtract $f(x^*)$ on both sides
    \begin{align}
         \mathbb{E}[f(x_0^{t+1})- f(x^*)] &\leq (1-\frac{n \mu \eta}{4})\mathbb{E}[f(x_0^{t})- f(x^*)] + L_{\max}^2 \eta^3 n^2 \sigma^2  
    \end{align}
    Solve this recursively, 
    \begin{align}
          \mathbb{E}[f(x_0^{T+1})- f(x^*)]\leq (1-\frac{n \mu \eta}{4})^T(f(x_0)-f(x^*)) + L_{\max}^2\eta^3n^2\sigma^2\sum_{j=0}^T(1-\frac{n\mu\eta}{4})^j. 
    \end{align}
    This implies (recalling that $x_0^{T+1} = x_n^T$):
\begin{align}
   \mathbb{E}  [f(x_n^T)-f(x^*)] &\leq
    (1-\frac{n \mu \eta}{4})^T(f(x_0)-f(x^*)) + L_{\max}^2\eta^3n^2\sigma^2\sum_{j=0}^{\infty}(1-\frac{n\mu\eta}{4})^j \nonumber\\
    &\leq(1-\frac{n \mu \eta}{4})^T(f(x_0)-f(x^*)) + L_{\max}^2\eta^3n^2\sigma^2(\frac{4}{n\mu \eta}) \nonumber\\
    &=(1-\frac{n \mu \eta}{4})^T(f(x_0)-f(x^*)) + \frac{4L_{\max}^2\eta^2n\sigma^2}{\mu}
    \label{eq:thm1_2}.
\end{align}
    Next, we show the proof under Assumption \ref{WGC_noise}. Substitute \eqref{eq:lem_rr_wgc} from Lemma \ref{lem_rr} into \eqref{eq:lem_descent_rr} in Lemma \ref{lem_descent} 
    \begin{align}
        \mathbb{E}[f(x_0^{t+1})] &\leq \mathbb{E}[f(x_0^{t})] - \frac{n\eta}{2}\mathbb{E}[\lVert \nabla f(x_0^t)\rVert^2] + \frac{L_{\max}^2\eta}{2} (4\eta^2 n^2 \alpha L \mathbb{E}[f(x_0^t) - f(x^*)] \nn \\
        & \qquad + 2 \eta^2 n^3 \mathbb{E}[\lVert \nabla f(x_0^t) \rVert^2] +  2 \eta^2 n^2 \sigma^2) \nn \\
        & = \mathbb{E}[f(x_0^{t})] - \frac{n \eta}{2}(1- 2L_{\max}^2 \eta^2 n^2)\mathbb{E}[\lVert \nabla f(x_0^t)\rVert^2]  \nn \\ 
        & \qquad + 2L_{\max}^2 \eta^3 n^2 \alpha L \mathbb{E} [f(x_0^t)-f(x^*)] + L_{\max}^2 \eta^3 n^2 \sigma^2 \nn \\
        & \stackrel{(d)}{\leq} \mathbb{E}[f(x_0^{t})] - \frac{n \eta}{4}\mathbb{E}[\lVert \nabla f(x_0^t)\rVert^2] + 2L_{\max}^2 \eta^3 n^2 \alpha L \mathbb{E} [f(x_0^t)-f(x^*)] + L_{\max}^2 \eta^3 n^2 \sigma^2 \nn \\
        & \stackrel{(e)}{\leq} \mathbb{E}[f(x_0^{t})] - \frac{n \mu \eta}{2} \mathbb{E} [f(x_0^t)-f(x^*)] + 2L_{\max}^2 \eta^3 n^2 \alpha L \mathbb{E} [f(x_0^t)-f(x^*)] \nn \\ &\qquad + L_{\max}^2 \eta^3 n^2 \sigma^2 \nn \\
        & = \mathbb{E}[f(x_0^{t})] - \frac{n \mu \eta}{2} (1- \frac{4L_{\max}^2 \eta^2 n \alpha L}{\mu})  \mathbb{E} [f(x_0^t)-f(x^*)] + L_{\max}^2 \eta^3 n^2 \sigma^2 \nn \\
        & \stackrel{(f)}{\leq} E[f(x_0^{t})] - \frac{n \mu \eta}{4}\mathbb{E}[f(x_0^t)-f(x^*)] + L_{\max}^2 \eta^3 n^2 \sigma^2, 
    \end{align}
    where in (d) we choose $\eta \leq \frac{1}{2L_{\max}n}$, (e) follows from Assumption \ref{def_pl}, and in (f) we choose $\eta \leq \frac{\mu}{2\sqrt{2} L_{\max} \sqrt{n \alpha L}}$. Then subtract $f(x^*)$ on both sides and solve the recursion similarly as above to obtain the final results.
\end{proof}
\subsection{Theorem \ref{thm:ig} Proof}
\begin{proof}
    We first show the proof under Assumption \ref{SGC_noise}. Substitute \eqref{eq:lem_ig_sgc} from Lemma \ref{lem_ig} into \eqref{eq:lem_descent_ig} in Lemma \ref{lem_descent}
    \begin{align}
    f(x_0^{t+1}) &\leq f(x_0^{t}) - \frac{n\eta}{2}\lVert \nabla f(x_0^t)\rVert^2 + \frac{L_{\max}^2\eta}{2}(2 \eta^2 n^3 \rho \lVert \nabla f(x_0^t) \rVert^2 + 2 \eta^2 n^3 \sigma^2) \nn \\
    & = f(x_0^t) - \frac{n  \eta}{2}(1-2L_{\max}^2 n^2 \eta^2 \rho) \lVert \nabla f(x_0^t) \rVert^2 + L_{\max}^2 \eta^3 n^3 \sigma^2 \nn\\
    &\stackrel{(a)}{\leq} f(x_0^t) - \frac{n \eta}{4} \lVert \nabla f(x_0^t) \rVert^2 + L_{\max}^2 \eta^3 n^3 \sigma^2 \nn \\
    &\stackrel{(b)}{\leq} f(x_0^t) -  \frac{n \mu \eta}{2} (f(x_0^t) - f(x^*)) + L_{\max}^2 \eta^3 n^3 \sigma^2,
    \end{align}
    where in (a) we choose $\eta \leq \frac{1}{2L_{\max}n\sqrt{\rho}}$, and (b) is based on Assumption \ref{def_pl}. Subtract $f(x^*)$ on both sides 
    \begin{align}
        f(x_0^{t+1}) - f(x^*) \leq (1-\frac{n\mu\eta}{2})(f(x_0^t)-f(x^*)) + L_{\max}^2 \eta^3 n^3 \sigma^2,
    \end{align} 
    then solve the recursion similarly to the proof of Theorem \ref{thm:rr} to obtain
    \begin{align}
        f(x_n^T) - f(x^*) \leq (1-\frac{n\mu\eta}{2})(f(x_0)-f(x^*))^T + \frac{2L_{\max}^2\eta^2 n^2 \sigma^2}{\mu}
    \end{align}
    In Lemma \ref{lem_ig}, we also require $\eta \leq \frac{1}{\sqrt{2}nL_{\max}}$. Therefore, the overall requirement is $\eta \leq \min \{\frac{1}{\sqrt{2}nL_{\max}}, \frac{1}{2L_{\max}n\sqrt{\rho}}\}$. Next, we show the proof under Assumption \ref{WGC_noise}. Substitute \eqref{eq:lem_ig_wgc} from Lemma \ref{lem_ig} into \eqref{eq:lem_descent_ig} in Lemma \ref{lem_descent}
    \begin{align}
    f(x_0^{t+1}) &\leq f(x_0^{t}) - \frac{n\eta}{2}\lVert \nabla f(x_0^t)\rVert^2 + \frac{L_{\max}^2\eta}{2}(4 \eta^2 n^3 \alpha L (f(x_0^t) - f(x^*)) + 2\eta^2 n^3 \sigma^2) \nn \\
    &\stackrel{(a)}{\leq} f(x_0^{t}) -n\mu \eta (f(x_0^t)-f(x^*)) + 2L_{\max}^2 \eta^3 n^3 \alpha L(f(x_0^t)-f(x^*))+ L_{\max}^2 \eta^3 n^3 \sigma^2,
    \end{align}
    where (a) is based on Assumption \ref{def_pl}. Subtract $f(x^*)$ on both sides and regroup terms
    \begin{align}
        f(x_0^{t+1}) - f(x^*) &\leq (f(x_0^t) - f(x^*)) [1-n \mu \eta(1-\frac{2L_{\max}^2\eta^2 n^2 \alpha L)}{\mu}] + L_{\max}^2 \eta^3 n^3 \sigma^2 \nn \\
         &\stackrel{(b)}{\leq} (1-\frac{n\mu\eta}{2})(f(x_0^t)-f(x^*)) + L_{\max}^2 \eta^3 n^3 \sigma^2, 
    \end{align}
    where in (b) we choose $\eta \leq \frac{\sqrt{\mu}}{2L_{\max}n\sqrt{\alpha L}}$. Hence, we require $\eta \leq \min \{\frac{1}{\sqrt{2}nL_{\max}}, \frac{\sqrt{\mu}}{2L_{\max}n\sqrt{\alpha L}}\}$ in this case. Then solve the recursion similarly to the proof of Theorem \ref{thm:rr}.
\end{proof}
\section{Convergence Rate of Nguyen et al. for RR on Over-Parameterized $\mu$-PL Functions}\label{sec:literature_app}
In this section, we provide more details on how the linear convergence rates are obtained from Nguyen et al. \cite[Theorem 1]{nguyen2021unified}. Under Assumption \ref{interp}\footnote{By taking $\sigma^*=0$ in Nguyen et al.'s Theorem~1.}, to obtain a relationship of the form \eqref{eq:recur_form}, they require a learning rate $\eta \leq \min\{\frac{1}{2L_{\max}n},\sqrt{\frac{3}{8}}{\frac{\mu}{L_{\max}^2 n}} \}$. Applying Lemma \ref{lem_recur} with $\frac{1}{\eta_1} = \frac{1}{2L_{\max}n}$ and $\frac{1}{\eta_2} =  \sqrt{\frac{3}{8}}{\frac{\mu}{L_{\max}^2 n}}$, the convergence rate obtained in this case is $\mathcal{\tilde{O}}(\max\{ \frac{L_{\max}}{\mu}n,\frac{L_{\max}^2}{\mu^2}n \})= \mathcal{\tilde{O}}(\frac{L_{\max}^2}{\mu^2} n)$. Similarly, under Assumption \ref{SGC}\footnote{By taking $\Theta=\rho-1$ and $\sigma=0$ in Nguyen et al.'s Theorem~1.}, the required learning rate in Nguyen et al. \cite[Theorem 1]{nguyen2021unified} is $\eta \leq \min \{ \frac{1}{\sqrt{3}L_{\max}n},\frac{\mu}{2 L_{\max} \sqrt{n} \sqrt{n+\rho-1}}\}$. Then a convergence rate of $\mathcal{\Tilde{O}}(\max\{ \frac{L_{\max}}{\mu} n , \frac{L_{\max}}{\mu^2}\sqrt{n}\sqrt{n+\rho-1}\})$ can be obtained by applying Lemma \ref{lem_recur} with $\frac{1}{\eta_1} = \frac{1}{\sqrt{3}L_{\max}n}$  and $\frac{1}{\eta_2} = \frac{\mu}{2L_{\max}\sqrt{n}\sqrt{n+\rho-1}}$. 
\end{document}